\documentclass[11pt]{article}
\usepackage{amsmath,amsthm,amssymb,amsfonts}
\usepackage{latexsym}
\usepackage{graphicx,psfrag,import}
\usepackage{fullpage}
\usepackage{framed}
\usepackage{verbatim}
\usepackage{color}
\usepackage{epsfig}
\usepackage{epstopdf}
\usepackage[colorlinks=true,linkcolor=blue]{hyperref}
\usepackage{geometry}
\usepackage{mathtools}
\usepackage{nonfloat}
\usepackage{enumerate}
\usepackage{multicol}
\usepackage{booktabs}
\usepackage{enumitem}
\usepackage{lineno}
\usepackage{parcolumns}
\usepackage[normalem]{ulem}
\usepackage{xr}
\usepackage{epstopdf}
\usepackage{mathrsfs}
\usepackage{subfigure}
\usepackage{caption}
\usepackage{comment}
\usepackage{authblk}
\usepackage{setspace}
\usepackage[colorinlistoftodos]{todonotes}
\usepackage{physics}
\usepackage{neuralnetwork}
\usepackage{soul}

\usepackage{cite}
\usetikzlibrary{matrix,chains,positioning,decorations.pathreplacing,arrows}
\usetikzlibrary{positioning,calc}

\usepackage{mwe}

\usepackage{tikz}      
\usetikzlibrary{shapes,automata,positioning,arrows,fit}
\usepackage{mathtools}  
  \usepackage{array}
 \tikzset{every node/.style={auto}}
 \tikzset{every state/.style={rectangle, minimum size=0pt, draw=none, font=\normalsize}}
  \tikzset{bend angle=7}
  
  \newcommand{\R}{\mathbb{R}}
  \newcommand{\Z}{\mathbb{Z}}
  
  \def\wt{\widetilde}

\def\RR{\mathcal{R}}
  \def\WW{\mathcal{W}}
    \def\PP{\mathcal{P}}
      \def\BB{\mathcal{B}}
  \def\ds{\displaystyle}
  \def\mfX{\mathfrak{X}}
  
  \def\II{\mathcal{I}}
  \def\etaa{h}
  \def\com{Y}

\usepackage{adjustbox}
\newcommand{\specialcell}[2][c]{\begin{tabular}[#1]{@{}c@{}}#2\end{tabular}}

\newcommand\y{\boldsymbol{y}}

\def\S{\mathcal{S}}
\def\Re{\mathcal{R}}
\def\C{\mathcal{C}}
\def\Z{\mathbb{Z}}

\def\G{\mathcal{G}}

\theoremstyle{plain}
\newtheorem{theorem}{Theorem}[section]
\newtheorem{lemma}[theorem]{Lemma}

\newtheorem{prop}[theorem]{Proposition}

\theoremstyle{definition}
\newtheorem{definition}[theorem]{Definition}
\newtheorem{example}[theorem]{Example}

\def\cre{\color{red}}

\title{On reaction network implementations of  neural networks}

\author[1]{David F. Anderson}
\author[2]{Badal Joshi}
\author[3]{Abhishek Deshpande}

\affil[1]{Department of Mathematics, University of Wisconsin-Madison, {\tt anderson@math.wisc.edu}.}
\affil[2]{Department of Mathematics, California State University San Marcos, {\tt bjoshi@csusm.edu}.}
\affil[3]{Department of Mathematics, University of Wisconsin-Madison, {\tt deshpande8@wisc.edu}.}

\begin{document}

\maketitle

\abstract{
This paper is concerned with the utilization of deterministically modeled chemical reaction networks for the implementation of (feed-forward) neural networks.  We  develop a general mathematical framework  and prove that the ordinary differential equations (ODEs) associated with certain reaction network implementations of neural networks have desirable properties including (i) existence of unique positive fixed points that are smooth in the parameters of the model (necessary for gradient descent), and (ii) fast convergence to the fixed point regardless of initial condition (necessary for efficient implementation).  We do so by first making a connection between neural networks and fixed points for systems of ODEs, and then by constructing reaction networks with the correct associated set of ODEs.  We demonstrate the theory by constructing a reaction network that implements a neural network with a smoothed ReLU activation function, though we also demonstrate how to generalize the construction to allow for other activation functions (each with the desirable properties listed previously).  As there are multiple types of ``networks'' utilized in this paper, we also give a careful introduction to both reaction networks and neural networks, in order to disambiguate  the overlapping vocabulary in the two settings and to clearly highlight the role of each network's properties.
}

\section{Introduction}

 There is a growing interest in synthetic chemical reaction networks that carry out some pre-determined task \cite{cappelletti2020stochastic, QSW2011, SSW2010, buisman2009computing,napp2013message,gopalkrishnan2016scheme,virinchi2017stochastic,virinchi2018reaction,poole2017chemical,singh2019reaction, p1ccrnJournal, LeaderElectionDIST, condon2020approximate}. The field that develops and analyzes these networks often goes by the name ``Computation with chemical reaction networks (CRNs).''  The tasks being carried out can range from the pedestrian, such as determining the minimum or sum of two numbers, to the complex. The goal of this style of work  is not to devise  methods that can match or exceed silicon based computers in terms of speed, but instead it is to develop methods of computation for environments in which silicon based computers cannot  currently go--for instance in the cellular environment. A particular type of (complex) computation now found ubiquitously in our daily technology is machine learning via neural networks, and so it is no surprise that there has been  recent work  on the development of chemical reaction network implementations of neural networks with a fixed set of parameters~\cite{kim2004neural, hopfield1984neurons, hjelmfelt1991chemical, poole2017chemical, MSMW2019, DavidS2020, blount2017feedforward, chiang2015reconfigurable}.  More generally, work focused in this context on understanding the connection between biochemical models and the physical mechanisms of information processing stretches back at least through the 1960s \cite{rossler1974synthetic, rossler1974chemical, vohradsky2001neural, mjolsness1991connectionist, mestl1996chaos, benenson2012biomolecular, bray1990intracellular, sugita1963functional, buchler2003schemes}.

 The papers we are aware of in the literature pertaining to   chemical reaction network implementations of neural networks focus on  particular constructions.  Hence, there is currently little mathematical theory developed that can be utilized in a general manner.  (An exception is \cite{poole2017chemical}, which develops the necessary theory for chemical Boltzmann Machines to be implemented via stochastic models of chemical reaction networks.)    Moreover, it is often simulation that is put forth as evidence to demonstrate the validity of a construction as opposed to rigorous proof.  Thus, these works are not mathematical in nature.  (This should not be taken as a criticism, as these papers were not \textit{meant} to focus on the mathematics.)  The major goal of this work, therefore, is to begin the  development of a mathematical framework for the construction of deterministically modeled reaction networks that implement neural networks and machine learning.  In particular,  the mathematical framework will allow us to prove that the  dynamical system associated to the constructed chemical reaction network will (i) implement a given neural network, and (ii) have certain desirable properties, briefly outlined below.

Some further details are called for before proceeding. In order to devise deterministically modeled chemical reaction networks that implement neural networks, the following broad strategy may be employed:
 \begin{enumerate}
     \item Fix a neural network with some choice of activation function, $\varphi$, and parameters (biases and weights), $\PP$.  Denote the output values of the neural network via $\Psi(d),$ where $d$ is an input (data).
     \item Determine a chemical reaction network $\{\S,\C,\Re\}$ for which the associated mass-action ODE system
     \begin{equation}\label{eq:786586554}
        \dot x(t) = f(x(t)),\quad x(0) = d,
     \end{equation}
     satisfies $F(x) = \Psi(d),$ where $F$ is some functional of the solution, $x$, to \eqref{eq:786586554} (note that the solution $x$, depends on $d$, the initial value).  In particular, it is natural to take the output to be the limiting steady state values of some ordered subset of the species,
     \[
        F(x) = \left(\lim_{t\to \infty} x_i(t)\right)_{i \in \mathcal I},
     \]
     where $\mathcal I$ is some index set.
 \end{enumerate}
 The above is the basic strategy of~\cite{blount2017feedforward}, in which they design a reaction network to learn the XOR function, and of~\cite{MSMW2019}. 
We note that a different modeling framework is used in \cite{DavidS2020},
 in which limiting values are found when certain counts go to zero (and remain there).  
 
 The basic strategy outlined above, i.e.~utilizing the limiting values of an initial value problem \eqref{eq:786586554} to represent the output of a neural network, is quite natural, but it leaves open a number of  questions that need to be addressed for a given construction:
 \begin{enumerate}
    \item When will the constructed reaction network admit limiting steady-states?
     \item Assuming limiting steady-state values exist, under what conditions will they be unique for a given choice of model parameters and for a given initial condition?    
     \item Assuming there are unique limiting steady states, when will they be smooth in the parameters (which is important for gradient descent and other optimization procedures)?
     \item How long will it take the model to converge?  In particular, could the time required to determine the output of the system depend strongly on the initial conditions?  
 \end{enumerate}
 We note that these are highly non-trivial questions in the present context as mass-action models of chemical systems are polynomial dynamical systems, and are known to exhibit myriad behaviors including chaotic behavior \cite{di1989limit}.

 In this article we develop a mathematical framework that is capable of  resolving the questions posed above.  Moreover, we utilize our framework to develop a chemical reaction network implementation of an arbitrarily sized  neural network with a smoothed ReLU activation function (see equation \eqref{eq:modifiedReLU} and Figure  \ref{fig:SmoothedReLU}).  Using our mathematical framework, we prove that this construction leads to a system that is exponentially reliable (i.e., the output of the system is unique and is smooth with respect to the parameters of the model, and the process converges exponentially fast) and converges from infinity in finite time (so the convergence time is uniformly bounded over all initial conditions).  See Definitions \ref{def:786766} and \ref{def:3245873} below for the precise meaning of these terms.    

The applications possible from neural network implementations of chemical reaction networks seem nearly limitless.  However, it is the view of these authors that this potential can only be achieved once a solid mathematical foundation is created upon which to build the necessary theory and, eventually, physical implementations--perhaps via DNA strand displacement~\cite{QSW2011, qian2011neural, cherry2018scaling}.  We therefore view this work as a starting point, with follow-up work focused on implementations of  neural networks that can perform gradient descent autonomously,  allowing us to relax the assumption of a fixed set of parameters,  in both supervised and unsupervised settings. Finally, while the focus of the current paper is on implementations of neural networks via deterministically modeled reaction networks, stochastic variants are possible as well.  In particular, stochastically modeled reaction networks will be the more natural choice whenever the goal is the approximation of distributions as opposed to functions \cite{poole2017chemical}. Study of such implementations is therefore another exciting avenue of future research.

We end the this section with a brief collection of some notation that will be used throughout this paper.
We  denote  the empty set by $\varnothing$.  
We  denote an arbitrary index set by  $\II$. 
We  use the notation $\dot \bigcup_{i \in \II} A_i$ to mean the union $\bigcup_{i \in \II} A_i$ where $A_i \cap A_j = \varnothing$ for all $i, j \in \II$ such that $i \ne j$. 
By {\em partition} of a set $S$, we mean a collection of nonempty subsets of $S$, $\{A_i \ne \varnothing : i \in \II\}$, such that $S = \dot \bigcup_{i \in \II} A_i$. 
For two vectors $u,v$, we will denote the Hadamard product, which is simply term-wise multiplication, via $\odot$.  That is, we have
\[
    (u\odot v)_i = u_i\cdot v_i.
\]
For a function $f:\R^c \to \R$ and a vector $u= (u_1,\dots,u_c)$ we denote by $\nabla_u f$ the vector whose $i$th component is $\displaystyle \frac{\partial f}{\partial u_i}$. 
For a vector valued function $f$, we denote by $f'(x)$  the vector whose $i$th component is $f_i'(x)$. 

The remainder of the paper is organized as follows.  Sections \ref{sec:RNs} and \ref{sec:neural_network} give  primers, included notation used in this paper, on reaction networks and neural networks, respectively.  As there are two distinct notions of networks in this paper, it is important to carefully separate the two.
In Section~\ref{sec:reaction_NN}, we present our main theoretical results pertaining to ODE implementations of neural networks.  In Section~\ref{sec:implement}, we demonstrate how to utilize our theoretical results to construct a reaction network that implements a given neural network with a fixed set of parameters and a smoothed ReLU activation function.  In Section~\ref{sec:example}, we provide a detailed example, including a demonstration of how to utilize our theory to implement  neural networks with different activation functions.

 \section{Reaction networks}\label{sec:RNs}
 
Reaction networks are graphical representations of interactions between different ``species.''  In this context, the word species may refer to different organisms (for example, if you are modeling the interactions among foxes and hares) or to different chemical compounds (for example, if you are modeling the dynamics of a biochemical process within a cell).  In this paper, we are primarily interested in the latter context and will also refer to reaction networks as ``chemical reaction networks,'' as is common.

\begin{definition}\label{def:CRN}
    A \textit{reaction network}, or \textit{chemical reaction network}, consists of a nonempty and finite set of species $\S$ and directed graph with vertices $\C$ and directed edges $\RR$ satisfying the following conditions:
    \begin{itemize}
        \item each vertex is a linear combination of the species over the non-negative integers;
        \item every species appears with a positive coefficient in at least one vertex;
        \item no two vertices are the same linear combination of the species;
        \item each vertex is connected by a directed edge to at least one other vertex;
        \item there are no directed edges from a vertex to itself.
    \end{itemize}
    Vertices of the reaction network are called \textit{complexes}, and directed edges are called \textit{reactions}.  If $\com, \widehat\com\in \C$  are two complexes and there is a directed edge from $\com$ to $\widehat \com$, we will write $\com \to \widehat \com\in \RR.$  We will often denote a reaction network via $\G = (\S,\C,\Re).$\hfill $\triangle$
\end{definition}

When considering general/theoretical systems, we will typically denote the species as $\S = \{X_1,\dots,X_n\}$, in which case our vertices/complexes are of the form 
\[
    \com = b_1 X_1+ \cdots + b_n X_n,  \text{ where } b_i \in \Z_{\ge 0} \text{ for each } i \in \{1,\dots,n\}.
\]
We will use the common slight abuse of notation by also associating a complex $\com \in \C$ with the vector in $\Z^n_{\ge 0}$ whose $i$th component is $b_i.$  Using this convention, we define the \textit{reaction vector} for a reaction $\com \to \widehat \com \in \RR$ as 
\[
    \zeta_{\com\to \widehat \com} = \widehat \com - \com \in \Z^n_{\ge 0}.
\]
When considering specific examples, we will use more suggestive notation for our species.
We present two examples to solidify the notation.  It is a common practice, which we use here, to specify a reaction network by writing all the reactions, since the sets $\S$, $\C$, and $\Re$ are contained in this description.

\begin{example}\label{example:CRNNotation1}
    Consider the following reaction network with two species, $\S = \{X_1,X_2\}$
    \begin{align*}
        X_1 + X_2 &\to 2X_2\\
        X_2 &\to X_1.
    \end{align*}
    Here the set of complexes/vertices is $\{X_1 + X_2, \ 2X_2, \ X_2, \ X_1\}$.
    For example, it could be that $X_1$ is an active form of a protein and $X_2$ is the inactive form and two actions can take place:  (i) an inactive protein can catalyze the inactivation of an active protein, and (ii) an inactive protein can spontaneously become active.  For another example, we could use the network to model disease spread, with $X_1$ representing healthy/susceptible individuals and $X_2$ representing those that are infected. 
    
    Whatever the modeling scenario is, the network is the same and consists of two species, four complexes (vertices), and two reactions. The associated reaction vectors are
    \begin{align*}
        \zeta_{X_1 + X_2 \to 2X_2} = \left[ \begin{array}{r} -1\\1 \end{array}\right] \quad \text{and} \quad  \zeta_{X_2 \to X_1} = \left[ \begin{array}{r} 1\\-1 \end{array}\right].
    \end{align*}\hfill $\square$    
\end{example}

\begin{example}\label{example:CRNNotation2}
    Consider the following reaction network with three species, $\S = \{X_1,X_2,X_3\}$
    \[
    0 \to X_1 + X_2, \quad X_1 \rightleftharpoons X_3 \leftarrow X_1 + X_3.
    \]
    In this example, molecules of $X_1$ and $X_2$ enter the system from outside of it via $0 \to X_1 + X_2$, $X_1$ can spontaneously convert to $X_3$ and vice versa via the two reactions $X_1 \rightleftharpoons X_3$, and $X_3$ catalyzes the removal of $X_1$ molecules via the reaction $X_1 + X_3 \to X_3$.\hfill $\square$
\end{example}

The reaction network tells us the constituent species of a model, the counts of each of the species required for each of the reactions to take place, and the counts of the products of each reaction.  Moreover, the reaction vectors give the net changes in the counts of the species due to the occurrence of the different reactions.  However, the reaction network does not determine the \textit{rates} at which the different reactions take place.  

A common modeling choice is to assume that the vector of concentrations of the species at time $t\ge 0$, denoted by $x(t) \in \R^n_{\ge 0}$, satisfies a system of the form
\begin{align}\label{eq:55550}
    \dot x(t) = \sum_{\com \to \widehat \com \in \Re} \lambda_{\com \to \widehat \com}(x(t)) \zeta_{\com\to \widehat \com},
\end{align}
where the enumeration is over all of the reactions and $\lambda_{\com \to \widehat \com}: \R^n_{\ge 0}\to \R_{\ge 0}$ is some function.  The set of functions $\Lambda = \{\lambda_{\com\to \widehat \com}\}$ is called the \textit{kinetics of the model}, and the most common form of kinetics, and the one we use throughout, is termed \textit{mass-action kinetics} in which 
\[
    \lambda_{\com \to \widehat \com}(x) = \kappa_{\com \to \widehat \com} \prod_{i = 1}^n x_i^{\com_{i}},
\]
for some choice of rate constant $\kappa_{\com\to\widehat \com} > 0$ and where $\com_{i}$ is the $i$th component of $\com$ viewed as a vector in $\Z^n_{\ge0}$.  When $\Lambda$ is mass-action kinetics, we say that $(\G,\Lambda)$ is a \textit{mass-action system}.   When mass-action kinetics is used, it is common to place the reaction rate constant next to the associated arrow in the graph,  $\com \xrightarrow{\kappa_{\com\to \widehat \com}} \com.$

\section{Neural networks}\label{sec:neural_network}

We give a basic introduction to the type of neural networks we consider in this paper--feed forward.  For more on neural networks, see~\cite{bishop1995neural, mitchell1997machine, murphy2012machine, Nielsen2015, shalev2014understanding}. Loosely, a neural network is a graph that gives a visual depiction of a certain type of mathematical function. 
The class of functions they can represent, which will be detailed below, have many parameters, and are ``universal'' in that they can be used to approximate any continuously differentiable function arbitrarily well~\cite{cybenko1989approximation,hornik1991approximation}. The power of neural networks comes from the fact that they can be ``trained'' from data, which simply means that the parameters of the function can be calibrated algorithmically so as to produce a final function capable of carrying out some pre-determined task (such as image recognition).

Below, we will first introduce the basic structure of a neural network.  Next, we will explain how each such graph, when combined with a choice of parameters and an  ``activation function,''  is simply a representation for a particular function.  We will call such a network, in which all parameters, together with the activation function, are fixed, a ``hardwired'' neural network.  Finally, we will discuss how neural networks can be trained by finding parameters for the network that minimize (at least locally) a desired cost function.  This minimization is often performed by a version of gradient descent and is termed \textit{backpropagation} in the field.

\tikzset{%
  every neuron/.style={
    circle,
    draw,
    thick,
    minimum size=1.1cm
  },
  neuron missing/.style={
    draw=none, 
    scale=4,
    text height=0.333cm,
    execute at begin node=\color{black}$\vdots$
  },
}
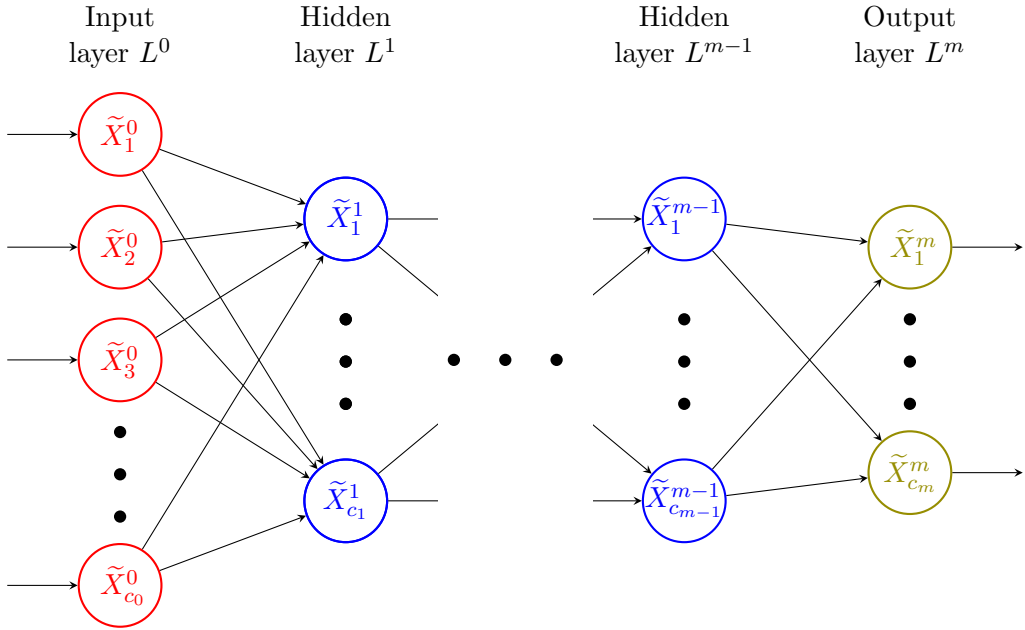
\begin{figure}
\begin{equation*}
\begin{tikzpicture}[x=1.5cm, y=1.5cm, >=stealth]

\foreach \m/\l [count=\y] in {1,2,3,missing,4}
  \node [every neuron/.try, neuron \m/.try,color=red] (input-\m) at (0,2.5-\y) {};

\foreach \m [count=\y] in {1,missing,2}
  \node [every neuron/.try, neuron \m/.try, color=blue] (hidden1-\m) at (2,2-\y*1.25) {};

\foreach \m [count=\y] in {1,missing,2}
\node [every neuron/.try, neuron \m/.try, color=blue ] (hidden1-\m) at (2,2-\y*1.25) {};

\foreach \m [count=\y] in {1,missing,2}
  \node [every neuron/.try, neuron \m/.try, color=blue ] (hidden2-\m) at (5,2-\y*1.25) {};

\foreach \m [count=\y] in {1,missing,2}
  \node [every neuron/.try, neuron \m/.try,color=olive ] (output-\m) at (7,1.5-\y) {};

\foreach \l [count=\i] in {1,2,3,c_0}
  \draw [<-] (input-\i) -- ++(-1,0)
    node [right=1.1cm] {{\cre $
    \wt X^{0}_{\l}$}};

\foreach \l [count=\i] in {1,c_1}
  \node [below=0.16cm] at (hidden1-\i.north) {{\color{blue} $\wt X^{1}_{\l}$}};

\foreach \l [count=\i] in {1,c_{m-1}}
  \node [below=0.14cm] at (hidden2-\i.north) {{\color{blue} $\wt X^{m-1}_{\l}$}};

\foreach \l [count=\i] in {1,c_m}
  \draw [->] (output-\i) -- ++(1,0)
    node [left=1cm] {{\color{olive} $\wt X^{m}_{\l}$}};

\foreach \i in {1,...,4}
  \foreach \j in {1,...,2}
    \draw [->] (input-\i) -- (hidden1-\j);

\foreach \i in {1,...,2}
  \foreach \j in {1,...,2}
    \draw [->] (hidden1-\i) -- (hidden2-\j);

    \foreach \i in {1,...,2}
    \foreach \j in {1,...,2}
    \draw [->] (hidden2-\i) -- (output-\j);

\node [align=center, above] at (0,2) {Input\\layer $L^0$};
\node [align=center, above] at (2,2) {Hidden \\layer $L^1$};
\node [align=center, above] at (5,2) {Hidden \\layer $L^{m-1}$};
\node [align=center, above] at (7,2) {Output \\layer $L^m$};
\node[fill=white,scale=4,inner xsep=0pt,inner ysep=5mm] at ($(hidden1-1)!.5!(hidden2-2)$) {$\dots$};
\end{tikzpicture}
\end{equation*}
\caption{The graphical structure of a neural network. The red, blue, and green nodes are input nodes, hidden nodes, and output nodes, respectively.  An arrow from one node to another is a representation of the direction of influence, i.e. an edge in $D$. The value of the ``tail'' node is input for computation of the value of the ``head'' node. }
\label{fig:NN}
\end{figure}

\subsection{Structure of a neural network}
Formally, a \textit{feedforward neural network} $G =(V,D)$ is a directed graph on a set of nodes $V$ and a set of directed edges $D \subseteq V \times V$, such that there is a partition of $V$ into {\em layers} $L^\ell$, $V = \displaystyle\dot\bigcup_{\ell=0}^m L^\ell$, with the property that $(\wt{X}',\wt{X}) \in D$ if and only if $\wt{X}' \in L^\ell$ and $\wt{X} \in L^{\ell+1}$ for some $\ell \in \{0, \ldots, m-1\}$. We will refer to the set $L^{\ell}$ as the {\em $\ell^{th}$ layer} of $G$, so $G$ has $m+1$ layers, and each $L^{\ell}$, with $0 \le \ell \le m,$ contains $c_\ell > 0$  nodes. The nodes in $L^0$ are referred to as {\em input nodes}, while those in $L^m$ as the {\em output nodes}. All  nodes in $\displaystyle\bigcup_{\ell=1}^{m-1} L^{\ell}$ are referred to as {\em hidden nodes} or {\em intermediate nodes}. We use {\em input layer}, {\em output layer}, and {\em hidden layer} to refer to each layer that contains the corresponding nodes. Note that we can partition $D$ as follows 
\begin{align} \label{eq:partD}
D = \dot \bigcup_{\ell: 1 \le \ell \le m} \dot \bigcup_{\wt X \in L^{\ell}} \dot \bigcup_{\wt X' :(\wt X',\wt X) \in D} (\wt X',\wt X).
\end{align}
For the sake of brevity,  for the remainder of the paper we will refer to feedforward neural networks simply as neural networks.

Indices can often become burdensome when working with neural networks.  Thus, we  minimized their use in the preceding explanation, and will continue to do so when possible.  That said, it will be useful to have  an enumeration and so we will denote the $j$th node in layer $\ell$ by $\wt X^\ell_j$. 
See Figure \ref{fig:NN}.

\subsection{A neural network  as a mathematical function}

We label each non-input node and each directed edge with a real number. A label for a non-input node is termed a \textit{bias} whereas a label for an edge is termed a \textit{weight}.  
Moreover, we associate an {\em activation function} with each non-input node, which will be described fully below.
We will call a neural network with such a labeling and a choice of activation function a {\em hardwired neural network}.
For each $\ell \in \{1,\dots,m\}$, we will denote by $\beta^\ell \in \R^{c_\ell}$ the vector whose $i$th component gives the bias for node $\wt X^\ell_{i}$, and will denote by $W^\ell \in \R^{c_{\ell} \times c_{\ell-1}}$ the matrix whose $(i,j)$th entry represents the weight of the edge between nodes $\wt X^{\ell-1}_j$ and $\wt X^{\ell}_i.$  Note that the ordering of the indices of $W^\ell$ seems backwards at first glance.  However, this ordering will make certain expressions slightly cleaner later, and is standard in the field.

We will use the notation $\BB$ for the assignment of node labels (biases) and $\WW$ for the assignment of edge labels (weights). That is, for each of $\ell\in \{1,\dots,m\},$ we have $\BB(\ell) = \beta^\ell$ and $\WW(\ell) = W^\ell$.  Collectively $\PP=(\BB,\WW)$ is an assignment of labels to $G=(V,D)$.   So long as we have also chosen an activation function $\varphi$, which will be described directly below, we may denote the resulting hardwired neural network via $(G,\PP,\varphi)$.

Let $\varphi: \R \to \R_{\ge 0}$ be a continuous, monotonic function, which is then extended to $\varphi: \R^c \to \R_{\ge 0}^c$ for $c \in \{2,3,\dots\}$ by letting $\left(\varphi(y)\right)_i = \varphi(y_i)$. We present a few examples of some  so-called {\em activation functions} $\ds \varphi: \R \to \R_{\ge 0}$. 
\begin{enumerate}
    \item $\varphi_1(y) = \frac{1}{1 + e^{-y}}$.  This sigmoid function is a bijection onto the interval $(0,1)$, and is used quite commonly. 
    See Figure \ref{fig:sigmoid}.
    
    \item $\varphi_2(y) = \max(0,y)$.  This function is termed the ReLU function (Rectified Linear Units). See Figure \ref{fig:ReLU}.
  
\item Let $\etaa \ge 0$ and define 
\begin{align}\label{eq:modifiedReLU}
    \varphi_3(y) = \frac{1}{2}\left(y + \sqrt{y^2 + 4\etaa}\right).
\end{align}
This function is a smoothed version of the ReLU function, while remaining strictly monotonic, and will play a key role in the present work. See Figure \ref{fig:SmoothedReLU}.
\end{enumerate}

\begin{figure}
\minipage{0.45\textwidth}
  \includegraphics[width=0.8\linewidth]{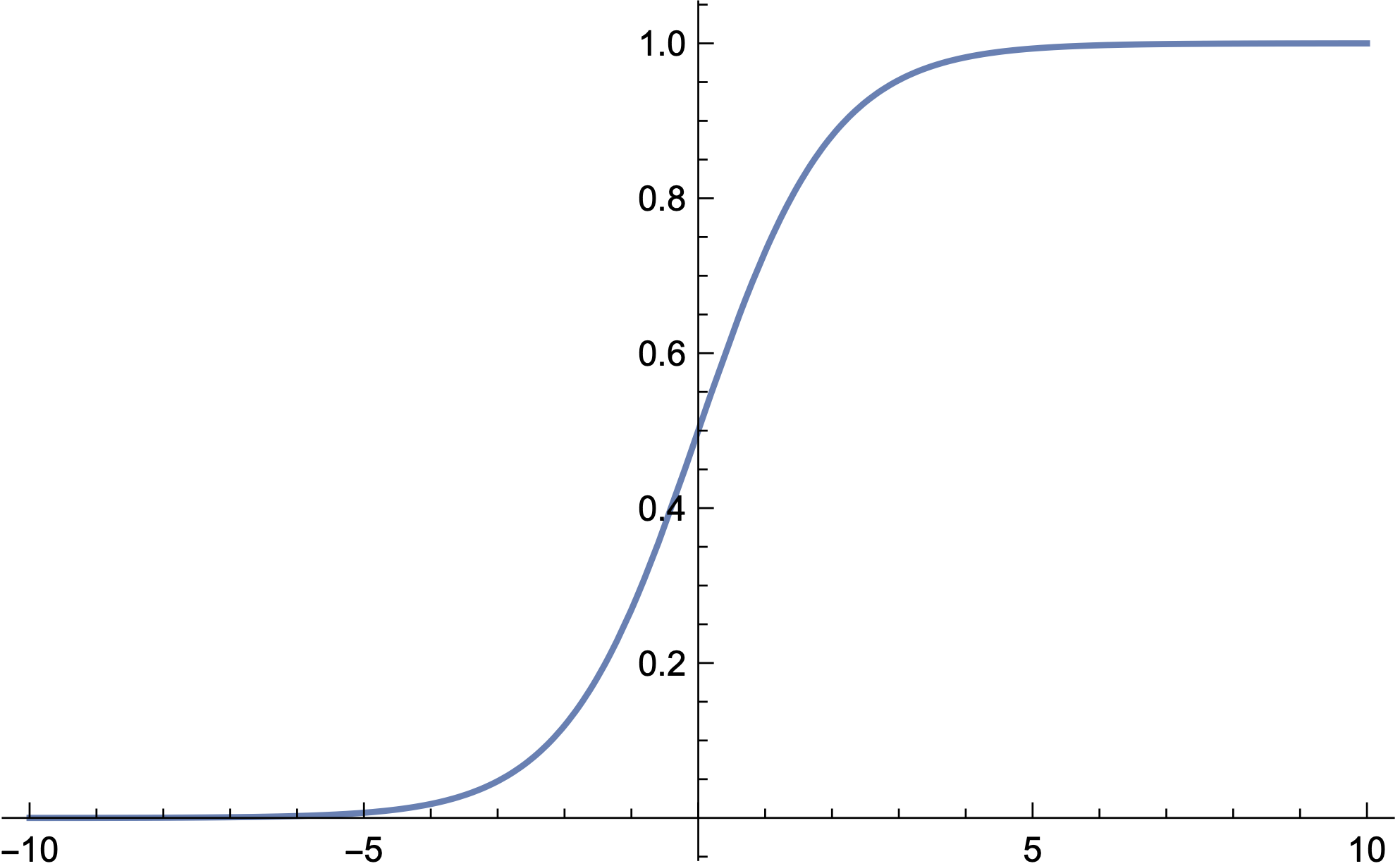}
  \caption{The function $\frac1{1+e^{-y}}$}\label{fig:awesome_image1}
  \label{fig:sigmoid}
\endminipage
\minipage{0.45\textwidth}
  \includegraphics[width=0.8\linewidth]{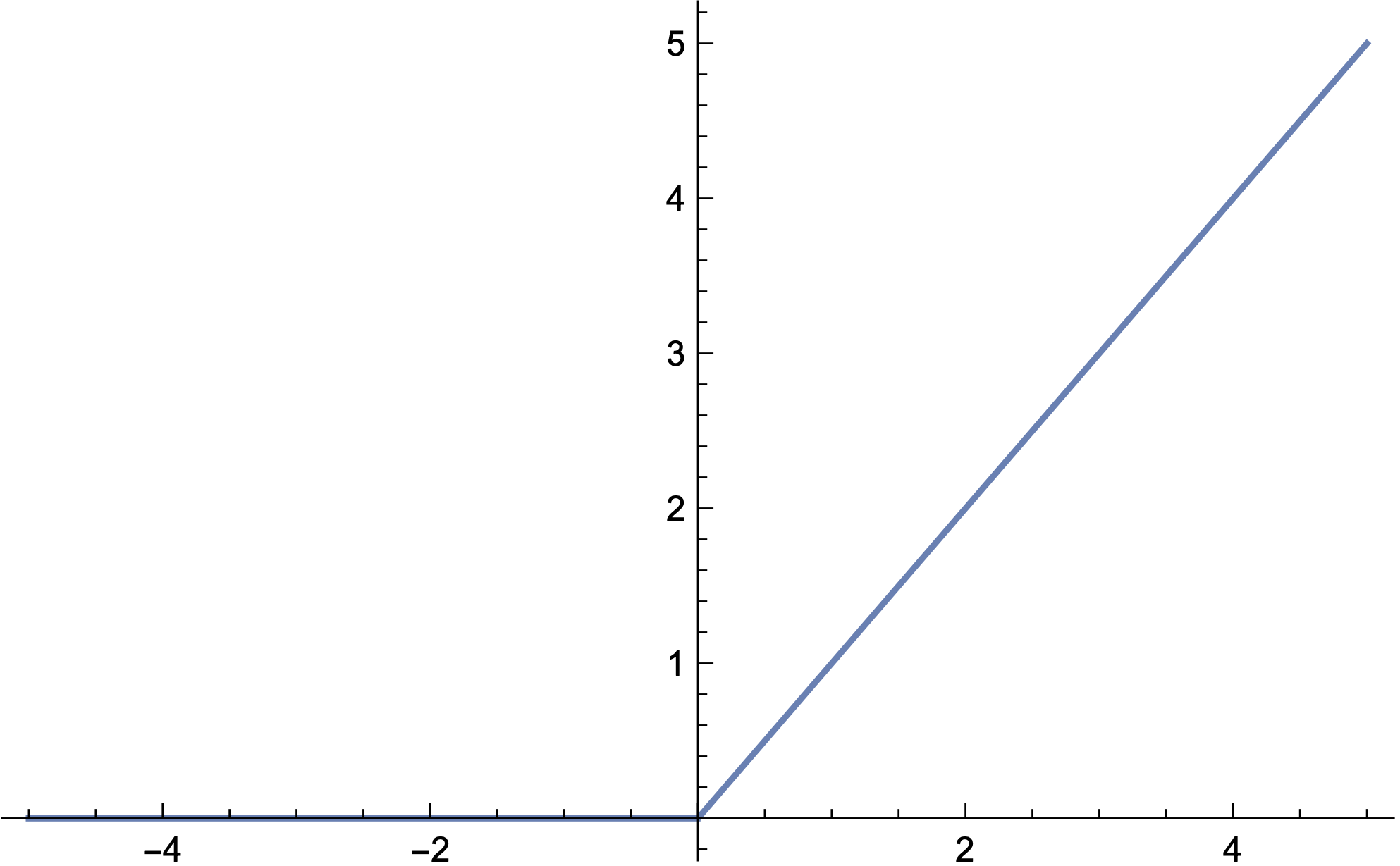}
  \caption{The ReLU activation function.}\label{fig:awesome_image2}
  \label{fig:ReLU}
\endminipage

\minipage{0.45\textwidth}%
  \includegraphics[width=0.8\linewidth]{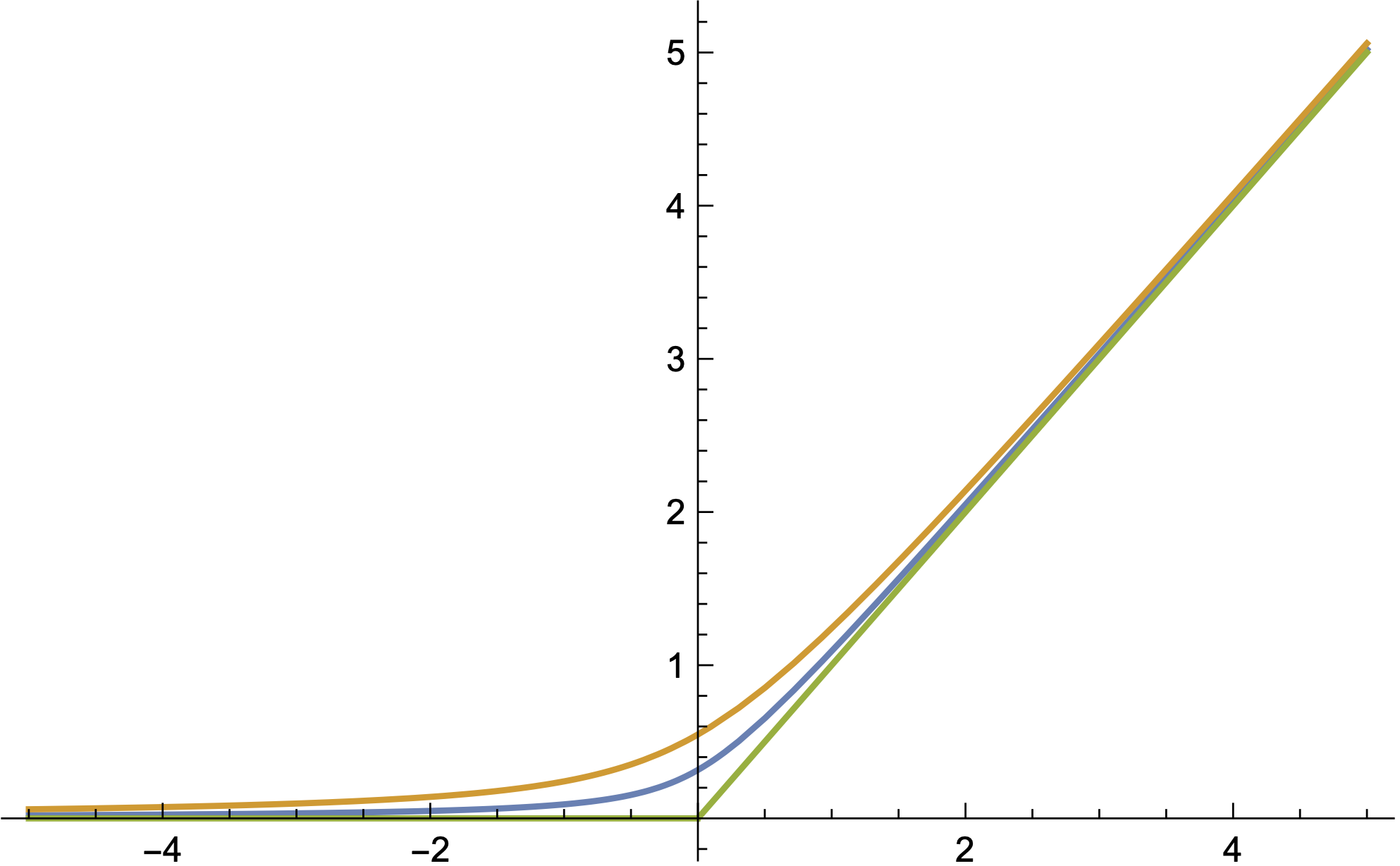}
  \caption{The function $\frac12\left( y + \sqrt{y^2 + 4\etaa}\right)$ for $\etaa = 0.3, 0.1, 0$. Note that the $\etaa=0$ case is the ReLU.}\label{fig:awesome_image3}
  \label{fig:SmoothedReLU}
\endminipage

\label{fig:activationfunctions}
\end{figure}

A pair of consecutive layers $L^{\ell-1}$ and $L^{\ell}$ along with all edges between the two layers, encode   a function $\psi^\ell: \R^{c_{\ell-1}} \to \R^{c_\ell}$ which is defined via 
\begin{align} \label{eq:func_psi}
\psi^\ell(y) = \varphi(W^\ell y + \beta^\ell).
\end{align}
Taking  compositions, a hardwired neural network is then simply a visual representation for  the function $\Psi_{(G,\PP,\varphi)}: \R^{c_0} \to \R^{c_m}_{\ge 0}$ defined via
\[
    \Psi_{(G,\PP,\varphi)} = \psi^m \circ \psi^{m-1} \circ \cdots \circ \psi^1.
\]
 Thus, the function associated with a neural network is simply a sequence of compositions that alternates between linear functions (via matrix multiplication and vector addition) and non-linear functions (via application of the activation function).

It is useful to provide a bit more notation before moving on.  Suppose that $d \in \R^{c_0}$ is the input to the function $\Psi_{(G,\PP,\varphi)}$ (or, equivalently, the function $\psi^1$).  We then define $a^0 = d$ and for $1 \le \ell \le m$ we define 
\begin{align}\label{eq:firststep}
    z^\ell(d) &= W^\ell a^{\ell-1}(d)+\beta^\ell, \quad \text{ and } \\ 
    a^\ell(d) &= \varphi(z^\ell(d)),\label{eq:secondstep}
\end{align}
recursively, where we recall that the $i$th component of $\varphi(z^\ell(d))$ is $\varphi(z^\ell_i(d)).$  The vector $a^\ell(d)$ is said to give the \textit{activations} of the nodes in the $\ell$th layer.  With these definitions we have that for any $\ell \in \{1,\dots,m\}$
\[
    \Psi_{(G,\PP,\varphi)}(d) = \psi^m\circ \cdots \circ \psi^\ell(a^{\ell-1}(d)).
\]
Moreover, note that $\Psi_{(G,\PP,\varphi)} = a^m,$ which is a useful compact notation for $\Psi_{(G,\PP,\varphi)}$.

\subsection{Learning from data}

Suppose now that we are given $N$ pieces of data of the form $(d,\tau(d)) \in \R^{c_0 \times c_m}$. For example, and to take a common example, $d\in \R^{784}$ could be the  values of the $28\times 28=784$ pixels in a gray-scale image of a hand-drawn number, and $\tau(d) \in \R^{10}_{\ge 0}$ could be the vector $e_i$ (the vector with a 1 in the $i$th digit and zeros elsewhere) if the image is that of a hand-drawn $i-1$.  Here $d$ is considered the input data and $\tau(d)$ is considered the ``truth.''    
We could then construct a neural network with $c_0 = 784$ and $c_m = 10$ simply by choosing (i) the number of hidden layers,  and how many nodes per layer, (ii) biases and weights, $\PP = (\BB,\WW)$, for the nodes and directed edges, and (iii) an activation function $\varphi$.  In such a manner, our hardwired function $\Psi_{(G,\PP,\varphi)}$ is determined.

At this point, we could ask how closely our function matches the ``truth'' by looking at some cost function.  Therefore, assume that we have a cost function of the form
\begin{equation}\label{eq:firstcostfunction}
    \text{Cost}(\PP)= \frac{1}{N} \sum_d C(d,\PP) = \frac{1}{N} \sum_d C(d),
\end{equation}
where the sum is over all the data and $C$ is a function giving a measure of how closely $\Psi_{(G,\PP,\varphi)}(d) = a^m(d)$ approximates $\tau(d)$. The second equality above points out that for notational convenience we will typically suppress the dependence of the parameters $\PP = (\BB,\WW)$ in $C$.  Some of the most commonly used  cost functions are given below:
\begin{enumerate}
    \item the quadratic cost function, in which case
    \begin{equation}\label{eq:i7087778}
    C(d) =  \frac12(\Psi_{(G,\PP,\varphi)}(d) - \tau(d))^2=  \frac12(a^m(d) - \tau(d))^2;
    \end{equation}
    \item the one-norm cost function, in which case
    \[
    C(d) =  \left| a^m(d) - \tau(d)\right|;
    \]
    \item the cross-entropy cost function, in which case
    \[
    C(d) = -\left[\tau(d) \ln(a^m(d)) + (1-(\tau(d)))\ln(1 - a^m(d)) \right].
    \]
\end{enumerate}
In this paper, we will take $C$ to be given by the quadratic cost function \eqref{eq:i7087778}. This choice of cost function does not play a significant role in the present work.

Of course, we did not specify how we chose our parameters $\PP = (\BB,\WW)$ for the model.  Supposing we choose them randomly somehow, there is no reason our function $\Psi_{(G,\PP,\varphi)}$ should be a good approximation for $\tau$ for the given data.  Therefore, we would like to find those parameters $\PP$ that minimize the cost function and to do so it is natural to use gradient descent.  Thus,
we need to be able to efficiently compute 
$\nabla_{\beta^\ell}\text{Cost}$ and $\frac{\partial \text{Cost}}{\partial W^\ell_{ij}}$
 for each appropriate value of $\ell, i$, and $j$. Because of the sum in \eqref{eq:firstcostfunction}, it is sufficient to compute the gradient of $C(d)$, and these can be computed as follows \cite{Nielsen2015} 
\begin{align}
\label{eq:allthederivatives}
\begin{split}
    \delta^L(d) &= \nabla_{a^m} C(d) \odot \varphi'(z^m(d))\\
    \delta^\ell(d) &= ((W^{\ell + 1})^T \delta^{\ell+1}(d))\odot \varphi'(z^\ell(d))\\
    \nabla_{\beta^\ell} C(d) &= \delta^\ell(d)\\
    \frac{\partial C(d)}{\partial W^\ell_{ij}} &= \delta_i^\ell(d) a_j^{\ell-1}(d)
    \end{split}
\end{align}
where $\nabla_{a^m}C(d)$ is the gradient of $C(d)$ with respect to $a^m$. For example, if $C$ is given by the quadratic cost function \eqref{eq:i7087778}, we have
\[
    \nabla_{a^m} C(d) = (a^m(d) - \tau(d)).
\]

\section{Neural networks and ODEs}
\label{sec:reaction_NN}

Fix a  hardwired neural network $G = (V,D)$ with parameters $\PP = (\BB,\WW)$, whose $\ell^{\text{th}}$ layer contains $c_\ell$ nodes, in which each node has  activation function $\varphi$. Let $W^\ell, a^\ell,$ and $\beta^\ell$ be as in the previous section.  

Now consider a system of ODEs defined recursively via 
\begin{align}\label{eq:5456897}
    x_i^0(t) &\equiv d_i, \hspace{1in} \text{ for some fixed } d \in \R^{c_0}_{ \ge 0},\\
\label{eq:870877807}
    \frac{d}{dt} x_i^\ell (t) &= f^\ell_i(x^{\ell-1}(t),x_i^\ell(t)), \quad \text{for $\ell \in \{1, \dots, m\}$},
\end{align}
where $x^\ell \in \R^{c_\ell}_{\ge 0}.$   Here we use $d\in \R^{c_0}_{\ge 0}$ to denote our initial condition as it represents the input ``data'' to the system.
Note that $x^{\ell-1}$ is acting as an external ``forcing function'' on $x^\ell.$
In particular, the system above has a natural feed-forward structure.   For  $r \in \{0,1,\dots,m\}$ we denote by $\mathcal{F}^r$ the subsystem of \eqref{eq:5456897}--\eqref{eq:870877807} consisting of only  those terms $x_i^\ell$ for which  $\ell \le r$.  Note that for any $1\le r \le m$, $\mathcal{F}^r$ contains $\mathcal{F}^{r-1}$ and that $\mathcal{F}^m$ is all of \eqref{eq:5456897}--\eqref{eq:870877807}.

\begin{definition}\label{def:implement}
    Suppose that for each fixed choice of $d\in  \R^{c_{0}}_{\ge 0}$  the system \eqref{eq:5456897}--\eqref{eq:870877807} has a unique solution $\{x^\ell\ :\ 1\le \ell \le m\}$ that satisfies 
    \[
        \lim_{t \to \infty} x^\ell(t) = \varphi(W^\ell a^{\ell-1}(d) + \beta^\ell) = a^\ell(d)\in \R^{c_\ell}_{\ge 0}
    \]
     for any choices of $x_i^\ell(0)\in \R_{\ge 0}$ for $\ell \ge 1$.
Then we say that the  system \eqref{eq:5456897}--\eqref{eq:870877807} {\em implements the neural network} $(G, \PP, \varphi)$. \hfill $\triangle$
\end{definition}

Note that in order for a system to implement a neural network according to the above definition, it is \text{not} enough for the system to simply convert  inputs,  $d$, to the correct outputs, $a^m(d) = \Psi_{(G,\PP,\varphi)}(d)$.  Instead, we require that the system calculates the activations for each node in the network, i.e., $a^\ell(d)$ for all $\ell \le m$, and do so for any choice of initial condition in layers $1$ through $m$.

\begin{example}\label{example:ReLU}
Consider a system \eqref{eq:5456897}--\eqref{eq:870877807} with
\begin{equation}\label{eq:000998990}
    f_i^\ell(x^{\ell-1},x^\ell_i) =  \etaa + \rho_i^\ell(x^{\ell-1}) x_i^\ell  - (x_i^\ell)^2,
\end{equation}
where
\begin{align}\label{eq:rho}
    \rho_i^\ell (x^{\ell-1}) = \left(W^\ell x^{\ell-1} + \beta^\ell\right)_i = \sum_{j=1}^{c_{\ell-1}} W^\ell_{ij}x^{\ell-1}_j + \beta^\ell_i.
\end{align}
We claim that the system \eqref{eq:5456897}--\eqref{eq:870877807} with this choice of $f_i^\ell$ implements a neural network with the smoothed  ReLU function \eqref{eq:modifiedReLU}.  This statement will be proved rigorously below once we have some additional mathematical machinery. \hfill $\square$
\end{example}

For a particular choice of $\ell$ and $i$, we can think of the one-dimensional system \eqref{eq:870877807} as simultaneously implementing both the linear updating step \eqref{eq:firststep} and evaluation with the activation function \eqref{eq:secondstep}  for node $i$ in layer $\ell$.  This observation motivates the following.

\begin{definition}
If the system \eqref{eq:5456897}--\eqref{eq:870877807} implements the neural network $(G,\PP,\varphi)$, then  \eqref{eq:870877807} is termed the \textit{activation system} for node $i$ in layer $\ell$.\hfill $\triangle$
\end{definition}

The following definition is added for completeness.
\begin{definition}
We will say that $y:\R_{\ge0} \to \R^n$ \textit{converges exponentially} to $\widehat y\in \R^n$, and will write $y(t)\xrightarrow{\text{exp}} \widehat y$ if there are $c,\etaa > 0$ for which $|y(t) - \widehat y| \le c e^{-\etaa t}$ for all $t \ge 0.$ \hfill $\triangle$
\end{definition}

The following definition characterizes some nice properties that  activation systems \eqref{eq:870877807} can have. 

\begin{definition} 
Consider the following one-dimensional system in which $y:\R_{\ge 0} \to \R^p$ is some forcing function,
\begin{equation}\label{eq:978786}
    \frac{d}{dt} x(t) = f(y(t),x(t)).
\end{equation}
\begin{enumerate}
    \item Let $q >0$. The system \eqref{eq:978786} is said to have \textit{$q$-polynomial decay} if for any compact set $\mathcal{K} \subset \R^p$ there is an $M>0$ and a constant $c>0$ such that when $y \in \mathcal{K}$ and $x > M$ we have
    \[
        f(y,x) \le -c x^{q}.
    \]
    
   \item 
    System \eqref{eq:978786} is said to be \textit{exponentially feed-forward} if for each  $\widehat y \in \R^p$ there is an $\widehat x\in \R$ such that $y(t)\xrightarrow{\text{exp}} \widehat y$ implies $x(t)\xrightarrow{\text{exp}} \widehat x$, assuming $x(t)$ exists for all $t \ge0$. \hfill $\triangle$
\end{enumerate}
\end{definition}

Thus, the system \eqref{eq:978786} has $q$-polynomial decay if it  decays faster than the solution to $\dot u = -cu^q$ when (i) the forcing function takes values that are not too large (quantified by $\mathcal{K}$) and (ii) the current value of the process is large (quantified by $M$).   Note that for $u(0)>0$, the solution to $\dot u = -c u^q$ converges from infinity  in finite time if $q>1$.  For completeness, we have proven this in Proposition \ref{prop:6587567} in Appendix \ref{appendix}. 

The usefulness of a system of the form \eqref{eq:978786} being exponentially feed-forward comes from the fact that we would like to be able to understand the long-term behavior of $\dot x = f(y(t),x(t))$ via an understanding of the long-term behavior of $\dot x = f(\widehat y, x(t))$.  We note with the following simple example that one is \textit{not} always able to do so.

\begin{example}\label{example:notff}
    Consider the system of the form \eqref{eq:978786} with 
    \[
       f(y,x) = \begin{cases}
       1 & \text{ if } y > 1\\
       -x & \text{ if } y \le 1
       \end{cases}.
    \]
    The system with $y(t) = 1 + e^{-t}$ satisfies $y(t)\xrightarrow{\text{exp}} 1$.  However, for this particular choice of $y(t),$ we have $y(t) > 1$ for all $t\ge 0$.  Thus, $x(t) = x(0) + t$, which does not converge to the fixed point of $\dot x = f(1,x)$, which is zero regardless of $x(0)$.
    \hfill $\square$
\end{example}

Given the discussion above, it will be useful to consider dynamical systems of the form
\[
    \dot x(t) = f(y,x(t)),
\]
where $y$ should be thought of as a (time-independent) collection of parameters, but now $x$ is allowed to be higher-dimensional.

\begin{definition}\label{def:786766}
  Suppose that $\dot x(t) = f(y,x(t))$ with $x(t) \in \R^n_{\ge 0}$ and $y\in \R^p_{>0}$ is a parametrized dynamical system such that for any choice of $x(0)\in \R^n_{\ge 0}$ and  $y\in \R^p_{> 0}$ the  system   has a unique solution. We will say that the  system   
  
\begin{enumerate}

\item is  {\em reliable} if there is a 
continuously differentiable function $\mfX:\R^p_{>0} \to \R^n_{>0}$ such that for any choice of $x(0)\in \R^n_{\ge 0}$, we have  $\lim_{t \to \infty} x(t) = \mfX(y)$;

\item {\em converges from infinity in finite time} if there is a compact set $\mathcal{K} \subset \R^n_{\ge 0}$ and a $T:\R^p_{>0} \to \R_{>0}$ such that $x(t) \in \mathcal{K}$ for any $t \ge T(y)$ and $x(0)\in\R^n_{\ge 0}$; 

\item is \emph{exponentially reliable} if it is reliable and there is a  $\lambda:\R^p_{>0} \to \R_{>0}$ such that 
\begin{equation*}
   \hspace{1.5in} \abs{x(t) - \mfX(y)} \le \abs{x(0) - \mfX(y)}e^{-\lambda(y) t}. \hspace{1.6in} \triangle
\end{equation*}
\end{enumerate}
\end{definition}
Note that the definition of \textit{reliable} does not rule out the existence of fixed points outside of $\R^n_{\ge 0}$.

The main question we have is the following: when can we conclude that the fully parametrized system \eqref{eq:5456897}--\eqref{eq:870877807}
has our desirable properties (reliability, convergence from infinity in finite time, and exponential reliability).  The following theorem shows that these properties follow from easily checked conditions on the functions $f_i^\ell.$  In the theorem below, the vector of parameters $y$ should be thought of as a  steady state value for $x^{\ell-1}(t).$

\begin{theorem}\label{thm:main}
 Consider the system \eqref{eq:5456897}--\eqref{eq:870877807}.  Suppose that for each $\ell \in \{1,\dots, m\}$ and $i \in \{1,\dots, c_{\ell}\}$  the dynamical system 
 \[
    \frac{d}{dt} x(t) = f_i^\ell(y,x(t)), \quad x(t) \in \R, \quad y\in \R^{c_{\ell-1}}_{\ge 0},
 \]
 is reliable.  Moreover, assume that 
\[
    \frac{d}{dt} x(t) = f_i^\ell(y,x(t))
\]
has $q$-polynomial decay for some $q > 1$ and is exponentially feed-forward.  Then the system \eqref{eq:5456897}--\eqref{eq:870877807}  converges from infinity in finite time and is exponentially reliable.
\end{theorem}

\begin{proof}
The proof proceeds by induction on $r$ for the systems $\mathcal F^r$,  where we remind the reader that the systems $\mathcal F^r$ are defined below \eqref{eq:5456897}--\eqref{eq:870877807}.  Consider the case $\ell =1$, where we have
\[
    \frac{d}{dt} x^1_i(t) = f_i^1(x^0,x^1_i(t)), \qquad \text{for $i \in \{1,\dots, c_1\}$}.
\]
Here, reliability of $x^1_i$ follows by our assumption.  The convergence of $x^1_i$ from infinity in finite time follows by the assumption of $q$-polynomial decay (compare with $\dot u = -c u^q$).  Finally, the exponential reliability of $x^1_i$ follows from the exponential feed-forward assumption (here $x^0 \xrightarrow{\text{exp}} x^0$ trivially).  Hence, the system $\mathcal F^1$   satisfies all the desired properties.

Now suppose the result holds for $\mathcal{F}^r$ with $r < m.$  Then  there is a compact set $\mathcal{K} \subset \R^{c_r}_{\ge 0}$ and a time $T>0$ so that $x^r(t) \in \mathcal{K}$ for all $t\ge T$, and moreover $x^r \xrightarrow{\text{exp}} \widehat x^r$.  Hence, by the assumption of $q$-polynomial decay, $x^{r+1}(t)$ converges from infinity in finite time, and we may conclude that the system $\mathcal{F}^{r+1}$ does as well.  Finally, by the exponential feed-forward assumption on layer $r+1$, together with the assumption that $\dot x_i = f_i^{r+1}(y,x_i(t))$ is reliable,  we may conclude that $\mathcal{F}^{r+1}$ is  exponentially reliable, and the proof is complete.
\end{proof}

We return to the activation system presented in Example \ref{example:ReLU}.
\begin{prop}\label{prop:our_network_NEW}
Consider the hardwired  neural network $(G,\PP,\varphi)$ and the system \eqref{eq:5456897}--\eqref{eq:870877807}  
with
\[
    f_i^\ell(x^{\ell-1},x^\ell_i) =  \etaa + \rho_i^\ell(x^{\ell-1}) x_i^\ell  - (x_i^\ell)^2,
\]
where
\[
    \rho_i^\ell (x^{\ell-1}) = \left(W^\ell x^{\ell-1} + \beta^\ell\right)_i = \sum_{j=1}^{c_{\ell-1}} W^\ell_{ij}x^{\ell-1}_j + \beta^\ell_i,
\]
and $\etaa >0.$ 
This system implements, in the sense of Definition \ref{def:implement}, the hardwired feed-forward neural network $(G,\PP,\varphi)$ where $\varphi$ is given as the smoothed ReLU function \eqref{eq:modifiedReLU}.  Moreover, the system  converges from infinity in finite time and is exponentially reliable.
\end{prop}

\begin{proof}
The fixed points of $\dot x = f_i^\ell(z,x(t))$ are
\[
    \frac{\rho_i^\ell(z) \pm \sqrt{ \rho_i^\ell(z)^2 + 4\etaa}}{2},
\]
which satisfy
\[
    \frac{\rho_i^\ell(z) - \sqrt{ \rho_i^\ell(z)^2 + 4\etaa}}{2} < 0 < \frac{\rho_i^\ell(z) + \sqrt{ \rho_i^\ell(z)^2 + 4\etaa}}{2}.
\]
Note that the strict inequalities follow from $\etaa >0.$
The positive equilibrium is continuously differentiable in the argument $z$.  Moreover,  for any $x(0) \in \R_{\ge 0}$, asymptotic stability follows from standard methods.  Hence, each of $\dot x = f_i^\ell(z,x(t))$ is reliable.

For each $\ell$ and $i$, the system $\dot x(t) = f_i^\ell(y(t),x(t))$ has 2-polynomial decay. Hence, to apply Theorem \ref{thm:main} and complete the proof we simply need to show that  $\dot x(t) = f_i^\ell(y(t),x(t))$ is exponentially feed-forward.

Thus, consider $\dot x(t) = f_i^\ell(y(t),x(t))$ and suppose that $y(t)\xrightarrow{\text{exp}} \widehat y$.  Denote
\[
    x^+ = \frac{\rho_i^\ell(\widehat y) + \sqrt{ \rho_i^\ell(\widehat y)^2 + 4\etaa}}{2} \quad \text{ and } \quad 
    x^- = \frac{\rho_i^\ell(\widehat y) - \sqrt{ \rho_i^\ell(\widehat y)^2 + 4\etaa}}{2}
\]
and let 
\[
    V(x) = \frac12 (x- x^+)^2.
\]
Then,  by adding and subtracting appropriately,
\begin{align*}
    \frac{d}{dt} V(x(t)) &= (x(t)-x^+) (\etaa + \rho_i^\ell(y(t)) x(t) - x(t)^2)  \\
    &= (x(t)-x^+) (\etaa + \rho_i^\ell(\widehat y) x(t) - x(t)^2) +  (x(t)-x^+) (\rho_i^\ell(y(t)) -  \rho_i^\ell(\widehat y) ) x(t)  \\
    &= - (x(t) - x^-)(x(t) - x^+)^2 + (x(t)-x^+) (\rho_i^\ell(y(t)) -  \rho_i^\ell(\widehat y) ) x(t).
\end{align*}
By assumption $y(t)\xrightarrow{\text{exp}} \widehat y$, and so by linearity we have that $\rho_i^\ell(y(t)) \xrightarrow{\text{exp}} \rho_i^\ell(\widehat y)$.  Moreover, standard methods can be used to show that $x(t)$ is uniformly bounded in time.  Combining the above allows us to conclude that
\[
    \frac{d}{dt} V(x(t)) \le a(t) - M V(x(t)),
\]
where $0 \le a(t) \le c e^{-\etaa t}$ for some $c, \etaa >0$.  Hence, by Gronwall's inequality, see Appendix \ref{appendix},
\begin{align*}
    \frac{1}{2} (x(t) - x^+)^2 & = V(x(t) \le \frac12 (x(t) - x(0))^2 e^{-Mt} +c  \int_0^t e^{-M(t-s)}e^{-\etaa s}ds\\
    &= \frac12 (x(t) - x(0))^2 e^{-Mt} + \frac{c}{M-\etaa} \left( e^{-\etaa t} - e^{-Mt}\right),
\end{align*}
where we can select $\etaa \ne M$ by taking $\etaa$ slightly smaller if need be.  Taking square roots shows that $x(t)\xrightarrow{\text{exp}} x^+$ as desired.
\end{proof}

\section{Reaction network implementation of a hard-wired neural network with a smoothed ReLU activation function}
\label{sec:implement}

This section is split into two parts.  In Section \ref{sec:preliminaries}, we give some preliminary definitions and concepts.  In Section \ref{sec:construction}, we give the explicit construction.

\subsection{Preliminaries}
\label{sec:preliminaries}

Consider a reaction network $\G = (\S,\C,\Re)$. It is convenient to separate the species set $\S$ into a disjoint union of dynamic and enzymatic species. 
\begin{definition}
    $X_i \in \S$ is said to be an \emph{enzymatic species} if $(\zeta_{\com \to \widehat \com})_i = 0$ for all $\com \to \widehat \com \in \Re$.  A species is said to be a \emph{dynamic species} if it is not an enzymatic species.\hfill $\triangle$
\end{definition}
Thus, an enzymatic species is one whose concentration is fixed for all time to its initial value, regardless of the initial value of the system. Enzymatic species are referred to as such because they facilitate reactions to occur, just like biological enzymes; higher availability of enzymes results in a proportional speedup of reactions. We will use the notation $\S_{dyn}$ and $\S_{enz}$ for the set of dynamic species and enzymatic species, respectively, and since any species can only be one or the other, $\S = \S_{dyn} ~\dot \bigcup~ \S_{enz}$.


\begin{example} \label{ex:enzy1}
Consider the reaction network
\begin{align}  
    X + Y + E &\xrightarrow{k_1} 2Y + E \nonumber \\
    Y + F &\xrightarrow{k_2} X + F. \nonumber
\end{align}
Here $\S_{dyn}=\{X,Y\}$ and $\S_{enz}=\{E,F\}$. \hfill $\square$
\end{example}

The concentrations of enzymatic species are time-invariant by definition, and so they satisfy the trivial ODE $de/dt=0$, where $e$ refers to the concentration of some enzyme $E$. This ODE obviously has the solution $e(t) = e(0)$ for all $t \ge 0$, independent of the dynamics of the other variables, and so it is without any loss of information, that we can withhold the ODEs for the enzymes from our description. We simply regard the initial values of the enzymes as parameters in the dynamical system. Thus we would say that {\em the} parametrized mass-action dynamical system associated to the network in Example \ref{ex:enzy1} is 
\begin{align}
\dot x &= -k_1exy + k_2f y \nonumber\\
\dot y &=\phantom{-} k_1exy - k_2f y, \label{eq:dynenzy1}
\end{align}
where we regard $e$ and $f$ as positive parameters similar to $k_1$ and $k_2$. 

An alternative approach is to remove all enzymatic species and to ``absorb'' their time-invariant concentration into the rate constant of the reaction. For instance, the network in Example \ref{ex:enzy1} is dynamically equivalent to the following network 
\begin{align}  
    X + Y &\xrightarrow{k_1 e} 2Y  \nonumber \\
    Y &\xrightarrow{k_2f} X , \nonumber
\end{align}
in the sense that both give rise to an identical system of differential equations \eqref{eq:dynenzy1}. 

Even though the former construction, in which enzymatic species are included in the model description may seem superfluous, it offers flexibility that will be found to be useful later when we construct reaction networks modularly, and then take unions of them.  In these situations species that were once enzymatic for one of the subnetworks can be dynamic for the resulting larger network.
This perspective will also be useful in later work when we change our outlook from a reaction network implementation of a hardwired neural network to a neural network capable of learning. For a preview, suppose that we add a reaction to the network in Example \ref{ex:enzy1}, so the resulting network is
\begin{align}  
    X + Y + E &\xrightarrow{k_1} 2Y + E \nonumber \\
    Y + F &\xrightarrow{k_2} X + F \nonumber \\
    Z + F &\xrightarrow{k_3} Z. \label{eq:bignet}
\end{align}
Then $F$ has lost its status as an enzyme and has been moved to the set of dynamic species. The  species partition for the new network is $\S_{dyn} = \{X, Y, F\}$ and $\S_{enz} = \{ E, Z\}$. Addition of the reaction $Z + F \to Z$ allows us to modulate the concentration of $F$ and therefore also the rate at which the reaction $Y+F \to X+F$ occurs.

The example is illustrative of some general properties, which we now state. 
A {\em subnetwork} of a reaction network $\G = (\S,\C,\Re)$ is a reaction network $\G' = (\S',\C',\Re')$ such that $\Re' \subseteq \Re$. It necessarily follows that $\S' \subseteq \S$, since every species in $\S'$ must participate in some reaction in $\Re'$ and therefore also in $\Re$, and by similar reasoning $\C' \subseteq \C$. 
While $\S_{dyn}' \subseteq \S_{dyn}$, the containment for enzymes runs backwards, i.e. $\S_{enz} \cap \S' \subseteq \S_{enz}'$. 
For example, the reaction network in Example \ref{ex:enzy1} is a subnetwork of the reaction network \eqref{eq:bignet}. The above-mentioned containments are easily checked to hold for this particular example. 

With the assumption of mass-action kinetics, and for any particular choice of reaction rate constants, a reaction network can be translated into a system of ODEs via \eqref{eq:55550}. Given this mapping, it is natural to say that a dynamical property of the parametrized ODE system is a property of the underlying reaction network itself.   We proceed by fixing some notation, which will allow us to translate Definition \ref{def:786766} to the reaction network setting. Let $\G = (\S,\C,\Re)$ be a reaction network with $\S = \S_{dyn}~ \dot \bigcup~ \S_{enz}$, and fix some (arbitrary) ordering of the dynamic species set $\S_{dyn}$. Let $n \coloneqq \abs{\S_{dyn}}$ and $x(t) \in \R^{n}_{\ge 0}$ denote the vector of concentrations of the dynamic species with respect to the ordering. We also arbitrarily order the set of parameters, which includes the reaction rate constants and the initial concentrations of enzymes. With the ordering, the parameters can be identified with a vector in $ y \in \R^p_{> 0}$ where $p \coloneqq \abs{\Re} + \abs{\S_{enz}}$.

\begin{definition}\label{def:3245873}
Suppose that $\dot x(t) = f(y,x(t))$ with $x(t) \in \R^n_{\ge 0}$ and $y\in \R^p_{>0}$ is a parameterized dynamical system obtained by applying mass-action kinetics to $\G = (\S,\C,\Re)$ with $\S = \S_{dyn}~\dot \bigcup~ \S_{enz}$, $n \coloneqq \abs{\S_{dyn}}$, and $p \coloneqq \abs{\Re} + \abs{\S_{enz}}$. Suppose that for any choice of $x(0)\in \R^n_{\ge 0}$ and  $y\in \R^p_{> 0}$ the  system $\dot x(t) = f(y,x(t))$ has a unique solution. We say that $\G$ is {\em reliable},  {\em converges from infinity in finite time}, or \emph{exponentially reliable} if the parameterized system $\dot x(t) = f(y,x(t))$ has those respective properties according to Definition~\ref{def:786766}.   
\hfill $\triangle$
\end{definition}

\subsection{Construction}
\label{sec:construction}

\begin{figure}
\begin{equation*}
\begin{tikzpicture}[x=1.5cm, y=1.5cm, >=stealth]

  \node [every neuron/.try, neuron 1/.try, color=red ] (input-1) at (5,2-1) {};
  
      \node  [above = 0cm of input-1] {$\wt{X}'$};

  \node [every neuron/.try, neuron 1/.try,color=green ] (output-1) at (7,2-1) {};
  
 \node  [above = 0cm of output-1] {$\wt{X}$};

\draw [->] (input-1) -- (output-1);

    \node [above right=-0.25cm and 0.8cm of input-1] {$w$};
     \node [left = -0.8cm of output-1] {$b$};

\end{tikzpicture}
\end{equation*}
\caption{\small{A single edge in a neural network, with one input and one output node.}}
\label{fig:23495790}
\end{figure}
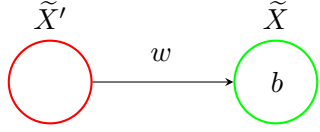

We will give the construction of a reaction network that implements a neural network with the smoothed ReLU activation function.  We will specifically design the network so that the ODE systems has $f_i^\ell$ as given in Example \ref{example:ReLU}. 

The construction will proceed in the following manner.  First, we build an explicit reaction network implementation of only a single edge, as depicted in Figure \ref{fig:23495790}, of the neural network, and describe the resulting parametrized ODE system. Second, we build on the previous step by giving a reaction network implementation of a single fixed node $\wt{X}$ in the neural network along with all of its inputs, and again describe the resulting parameterized ODE system.  Finally, we describe the reaction network implementation of the entire neural network, which results in the parametrized ODE system in~\eqref{eq:5456897}--\eqref{eq:870877807}.  For the sake of readability, we will limit the amount of enumeration utilized in our construction.

\begin{table}
\renewcommand{\arraystretch}{1.25}
  \centering
   \begin{adjustbox}{max width=\textwidth}
 \begin{tabular}{|c|c|c|}
 \hline
\specialcell{Which aspect of neural network \\ is implemented chemically?} & \specialcell{Chemical implementation of a \\ single directed edge $(\wt{X}' , \wt{X})$ \\ of the neural network}  & \specialcell{Which term results in the \\ ODE for the species $X$?} \\ 
 \hline
 \hline
\specialcell{Closeness to \\ ReLU} & $H \longrightarrow H + X$ &  $\etaa$ \\
 \hline
 \specialcell{Input $\wt{X}'$ and weight \\ of the edge $(\wt{X}',\wt{X})$} & \specialcell{$X' + W^+ + X \longrightarrow X' + W^+ + 2X $ \\ $X' + W^- + X \longrightarrow X' + W^- $} & \specialcell{$(w^+ - w^-)x' x$, \\ where $w \coloneqq w^+ - w^-$ \\ implements the edge weight} \\
   \hline
 \specialcell{Additive node bias \\ of $\wt{X}$} &    \specialcell{$ B^+ + X \longrightarrow  B^+ + 2X $ \\ $B^- + X \longrightarrow B^- $} & \specialcell{$(b^+ - b^-)x$, \\ where $b \coloneqq b^+ - b^-$ \\ implements the node bias}\\
     \hline
  \specialcell{$q$-polynomial decay \\ Stability/convergence from $\infty$} &   \specialcell{$qX \longrightarrow X $ \\ $(q > 1)$} & $(q-1)x^q$  \\
\hline
 \end{tabular}
 \end{adjustbox}
 \vspace{0.2cm}
 \caption{Components of an elementary reaction network -- chemical implementation of a single directed edge $(\wt{X}' , \wt{X})$ along with nodes $\wt{X}'$ and $\wt{X}$ of the neural network. A neural network is naturally viewed as a disjoint union of its edges, which allows putting together a chemical implementation as an appropriate union of elementary reaction networks.}
 \label{table:elementaryrxnnet}
 \end{table}

\vspace{.1in}

\noindent \textbf{Step 1.}  
The first step of the process, producing the reactions necessary for the implementation of a single edge, is carried out in Table \ref{table:elementaryrxnnet}. 
The species sets for this particular reaction network are $\S_{dyn} = \{X\}$, and $\S_{enz}=\{H, W^+, W^-, B^+, B^-, X'\}$. 
The associated mass-action ODE system is one-dimensional, in the variable $x$, and if we assume all reactions occur with a rate constant of 1, is
\begin{align} \label{eq:nxx'ode}
   \frac{d}{dt} x(t) = \etaa + \left(\left(b^+ - b^-\right) + (w^+ - w^-) x'\right) x(t) - (q-1)x(t)^q. 
\end{align}
We will assume from here on that $q=2$ and note that it is easy to make the necessary changes in the description for a general value different from $2$.  Note that when $q \ne 2$ the resulting activation function will be different than the smoothed ReLU.

\vspace{.1in}

\noindent \textbf{Step 2.}  
For the second step, we implement via reaction network the neural network depicted in Figure \ref{fig:34579988}, which now simply consists of $\wt{X}$ along with all its inputs.  For this particular node, we assume there are $c > 0$ inputs.  The construction proceeds by simply taking the union over the $c$ edges $(\wt X_i',\wt X)$ of the reaction networks described in Step 1.  After this union, we once again have that $X$ is the only dynamic species and the mass-action ODE for its concentration is given by
\begin{align}
       \frac{d}{dt} x(t) & = \etaa + \left(\left(b^+ - b^-\right) + \sum_{i=1}^c (w_{x'_i,x}^+ - w_{x'_i,x}^-) x'_i \right) x(t) - x(t)^2 = \etaa + \rho_x x(t) - x(t)^2,  \label{eq:dyn2}
\end{align}
where $\rho_x$ is defined by the equation above (and is analogous to $\rho_i^\ell$ from \eqref{eq:rho}).
Note that the above corresponds with the equations in Example \ref{example:ReLU}.
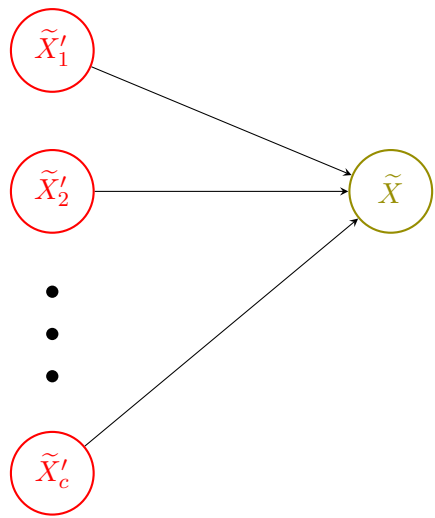
\begin{figure}
\begin{equation*}
\begin{tikzpicture}[x=1.5cm, y=1.5cm, >=stealth]

\foreach \m [count=\y] in {1,2,missing,3}
  \node [every neuron/.try, neuron \m/.try, color=red ] (input-\m) at (5,3.5-\y*1.25) {};


\foreach \l [count=\i] in {1,2,c}
  \node [above=-0.9cm] at (input-\i.north) {{\cre $\wt X'_{\l}$}};
  

  \node [every neuron/.try, neuron 1/.try,color=olive ] (output-1) at (8,2-1) {};

 \node  [above=-0.85cm of output-1] {{\color{olive} $\wt{X}$}};

    \foreach \i in {1,2,...,3}
    \draw [->] (input-\i) -- (output-1);
    


\end{tikzpicture}
\end{equation*}
\caption{Step two: node $\wt{X}$ along with all its $c$ inputs.}
\label{fig:34579988}
\end{figure}

\vspace{.1in}

\noindent \textbf{Step 3.}
The third step is to construct the final network by taking the union of the  construction described in the second step over all non-input nodes $\wt{X}$.  In terms of dynamical systems, this constitutes taking a union of the systems of ODEs given by \eqref{eq:dyn2},
with the appropriate indices applied to the variables and parameters. The final system of equations appears in \eqref{eq:5456897}--\eqref{eq:870877807} and in Example \ref{example:ReLU}. The entire system is repeated here for convenience of the reader. 
\begin{align*}
    x_i^0 &= d_i, \hspace{1in} \text{ for some fixed } d \in \R^{c_0}_{ \ge 0},\\
    \frac{d}{dt} x_i^\ell (t) &= \etaa + \left(\sum_{j=1}^{c_{\ell-1}} W^\ell_{ij}x^{\ell-1}_j(t) + \beta^\ell_i \right) x_i^\ell(t)  - (x_i^\ell(t))^2, \quad \text{for $\ell \in \{1, \dots, m\}$}. 
\end{align*}
Note that many species that were enzymatic in a particular network, for example the species associated with the terms $x'_i$ in the second step, are dynamic species in the final model.

\section{An example}
\label{sec:example}

In this section, we provide an example  to visually demonstrate several aspects of our theory and our constructions.  The focus of this paper was not on training a network--that will be the focus of our next work.  Instead, in this paper we focused on the different qualitative properties of possible constructions, as detailed in Definitions \ref{def:786766} and \ref{def:3245873}, and so this example will primarily share that focus.  We will  showcase how the limiting values of the ODE associated with  a reaction network that implements the modified ReLU activation function, as detailed in Section \ref{sec:implement}, match precisely with the more standard implementation of the neural network via direct use of the activation function \eqref{eq:modifiedReLU}.  Moreover, we will  demonstrate the fast convergence of the ODE, a property we have proven to hold in Proposition \ref{prop:our_network_NEW}.
Next, we will demonstrate the flexibility of the developed theory by chemically implementing a different activation function: one that grows like $\sqrt{y}$, as $y\to \infty$ (as opposed to linear growth in the case of ReLU), and converges to 0, as $y \to -\infty$.  This new implementation will still satisfy the conditions of Theorem \ref{thm:main}, and hence still enjoy the properties of Definitions \ref{def:786766} and \ref{def:3245873}. Finally, we will explain how any activation function with growth of the form $y^{1/k},$ as $y \to \infty$, for any integer $k \ge 1$, can likewise be implemented chemically. 

As it is a standard example in the field, we utilize the MNIST dataset of handwritten digits \cite{lecun1998mnist}. See Figure \ref{fig:digit} for four representative images from this dataset.  These images have $784 = 28\times 28$ pixels, and the task of the neural network is to take a grayscale image of such a hand drawn digit, and correctly identify the digit.  For example, we want the output to correctly identify the images in Figure \ref{fig:digit} as 8, 2,  6, and 7, respectively.  

\begin{figure}
    \centering
    \includegraphics[width=1.0in,height=1.0in]{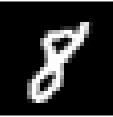}\qquad \includegraphics[width=1.0in,height=1.0in]{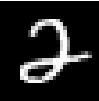}\qquad \includegraphics[width=1.0in,height=1.0in]{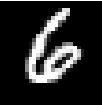}\qquad \includegraphics[width=1.0in,height=1.0in]{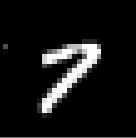}
    \caption{Representative examples of hand-drawn images from the MNIST database \cite{lecun1998mnist}.}
    \label{fig:digit}
\end{figure}

The input to the neural network can be regarded as a single vector of size 784.  Further, it is natural to choose the number of output nodes to be 10, with each node representing a different digit from the set $\{0,1,\dots,9\}$.  To complete the specification of the structure of the neural network, we will, somewhat arbitrarily, choose to have a single hidden layer with 40 nodes.  Therefore, our neural network has:
\[
c_0 = 784, \qquad c_1 = 40, \qquad c_2 = 10.
\]
As already mentioned, we will utilize the construction detailed in Section \ref{sec:construction}, yielding a smoothed ReLU \eqref{eq:modifiedReLU} as our activation function, and we will choose $h = 1$ as our smoothing parameter.

We will now clearly specify how we implemented our neural network in Matlab.  For the sake of reproducibility, we first set the seed of our random number generator by using the command ``rng(1234).''  We used this seed for every computation we are reporting in this section.  We then initialized our weights and biases randomly by utilizing scaled Gaussians via the following commands:
\begin{verbatim}
    W1 = (1/sqrt(c0))*randn(c1,c0);
    W2 = (1/sqrt(c1))*randn(c2,c1);
    beta1 = randn(c1,1);
    beta2 = randn(c2,1);
\end{verbatim}
In order to ensure that we use exactly the same random variables as does the reaction network implementation, we also defined an initial condition via the command 
\begin{verbatim}
    x00=10*rand([50,1]);
\end{verbatim}
as that call is necessary in our reaction network implementation in which each of the hidden and output nodes is a dynamic variable and therefore utilizes an initial condition.  While present in the code, this term is not used in the standard neural network implementation.

We utilized a quadratic cost function \eqref{eq:i7087778} in which the ``truth,'' denoted $\tau(d)$, was a vector with a 10 in the place of the true digit (i.e., if the digit represented by $d$ is zero, then $\tau(d)$ has a 10 in the 1st component, if the digit represented by $d$ is 1, then $\tau(d)$ has a 10 in the 2nd component, etc.), and has ones in all other components.
In order to implement gradient descent, we used a learning rate of $\eta = 0.1$ so that after each iteration of the neural network, we update our parameters via
\begin{align*}
    \beta^{\ell} &\leftarrow \beta^{\ell} - \eta \nabla_{\beta^{\ell}} \text{Cost}\\
    W^\ell_{ij} &\leftarrow W^\ell_{ij} - \eta \frac{\partial \text{Cost}}{\partial W^\ell_{ij}},
\end{align*}
for appropriate $\ell, i$, and $j$.  In order to estimate the derivatives above, we utilized stochastic gradient descent by using a batch of  300 randomly selected elements from the first 60,000 entries in the MNIST dataset.  The specific call we used in our Matlab code was
\begin{verbatim}
    Vals=randperm(60000,BatchSize);
\end{verbatim}
where BatchSize had been set to 300.  
See Figure \ref{fig:smoothedReLUPlots} for (i) the estimate of the cost function, and (ii) the number correctly predicted, out of the randomly chosen batch of 300, by the neural network over 1000 iterations of the learning process. Note that near the end of the 1000 iterations, the neural network is correctly identifying  just over 95\% of the digits.
\begin{figure}
    \centering
    \includegraphics[width=3in]{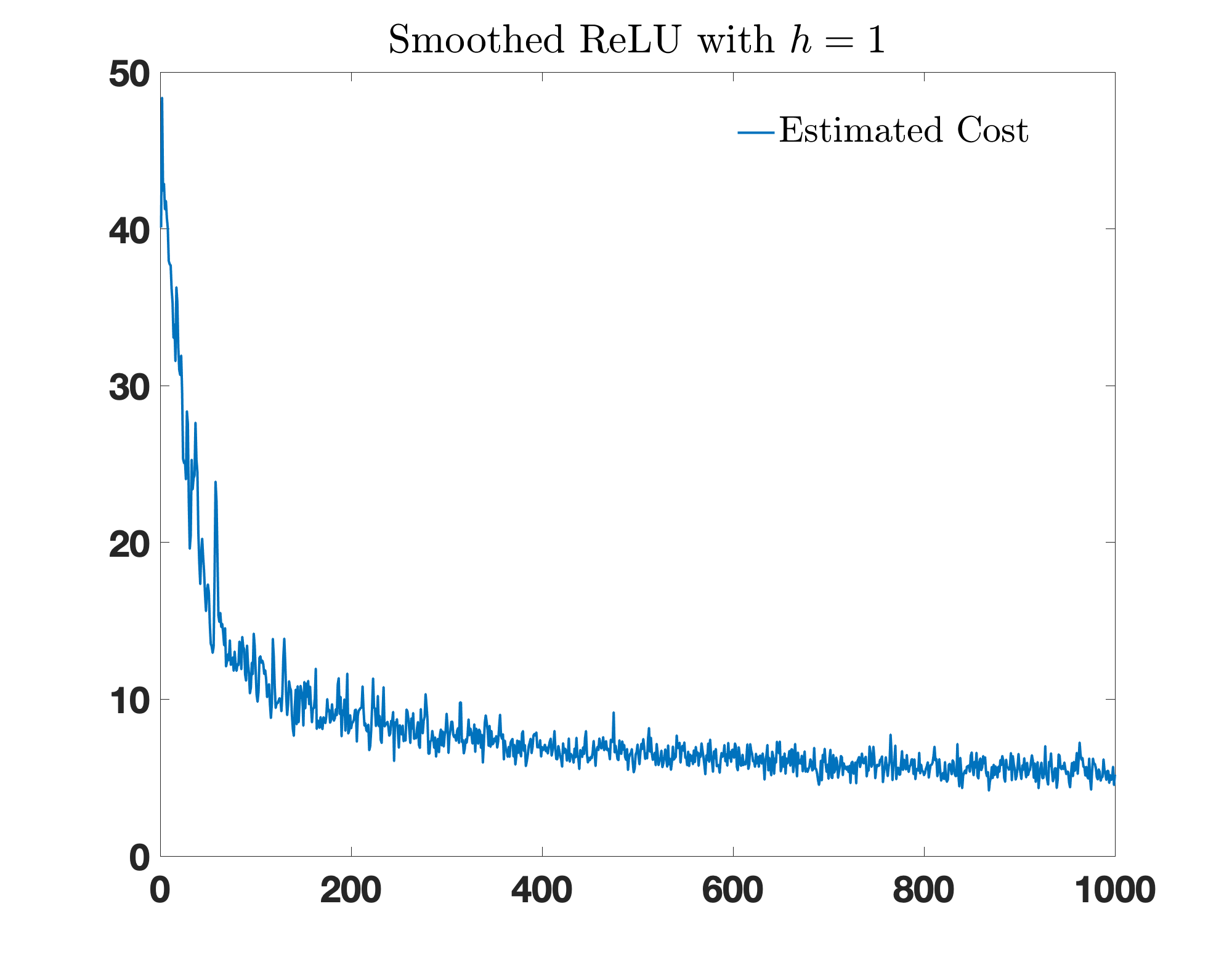}\quad \includegraphics[width=3in]{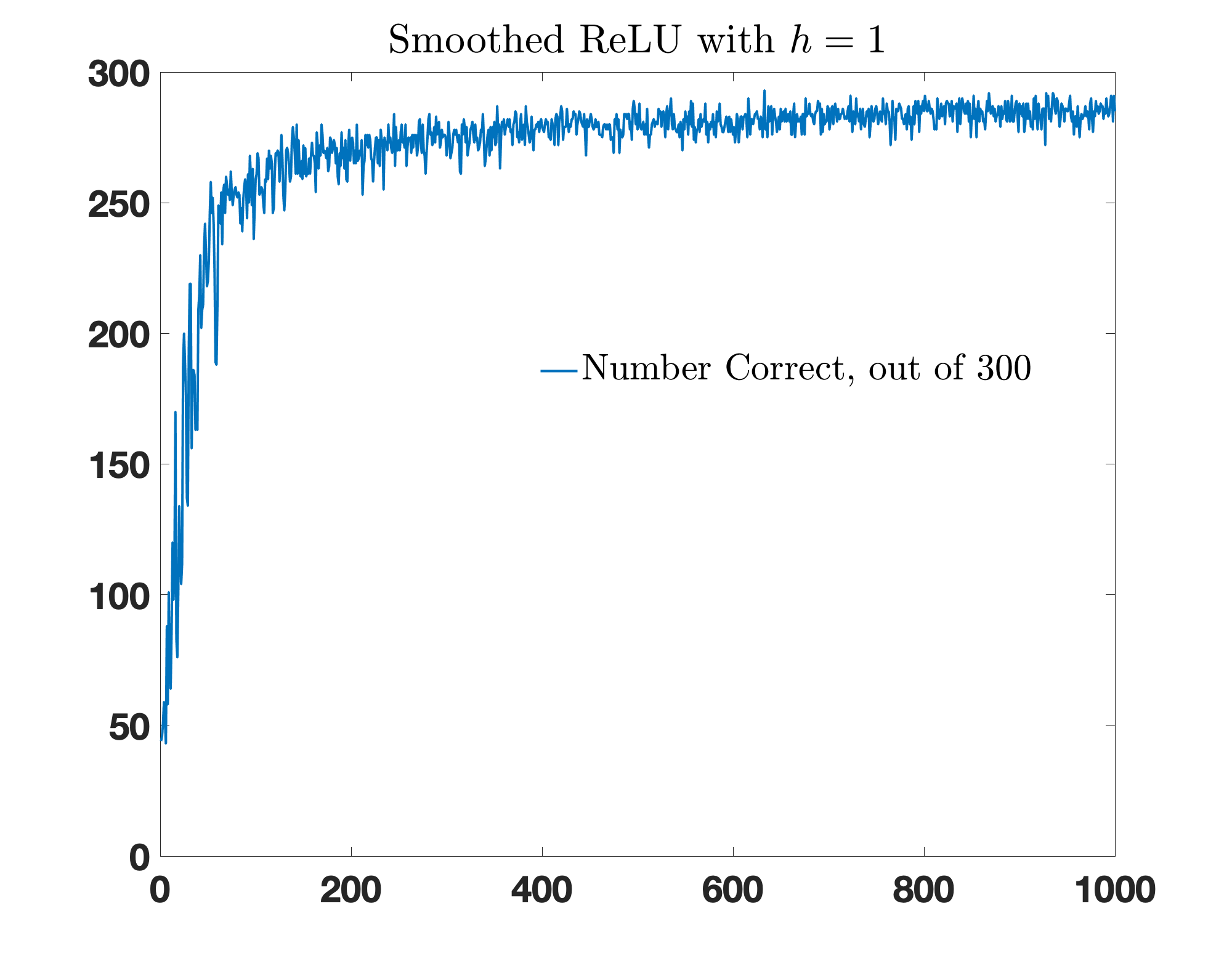}
    \caption{Performance of the smoothed ReLU cost function with $h=1$.  Left image: estimate of the cost function over each iteration of the neural network (from 300 randomly selected elements from the MINST dataset).  Right image:  total number of images from the 300 whose digits were correctly identified.  For each image the $x$-axis represents the iteration number of the learning process.}
    \label{fig:smoothedReLUPlots}
\end{figure}
For the sake of comparison, in Figure \ref{fig:ReLUPlots} we give similar plots for the standard ReLU activation function (i.e., taking $h=0$).  Now the neural network correctly identifies around 88\% of the digits.   The superiority of the smoothed version of the ReLU activation function was apparent in nearly all the seeds of the random number generator that we tried (data not shown).  The precise reason for the superiority of the smoothed version of the ReLU activation function in the present setting is unclear to us, though perhaps the lack of a zero derivative for $y < 0$ is playing a role.
\begin{figure}
    \centering
    \includegraphics[width=3in]{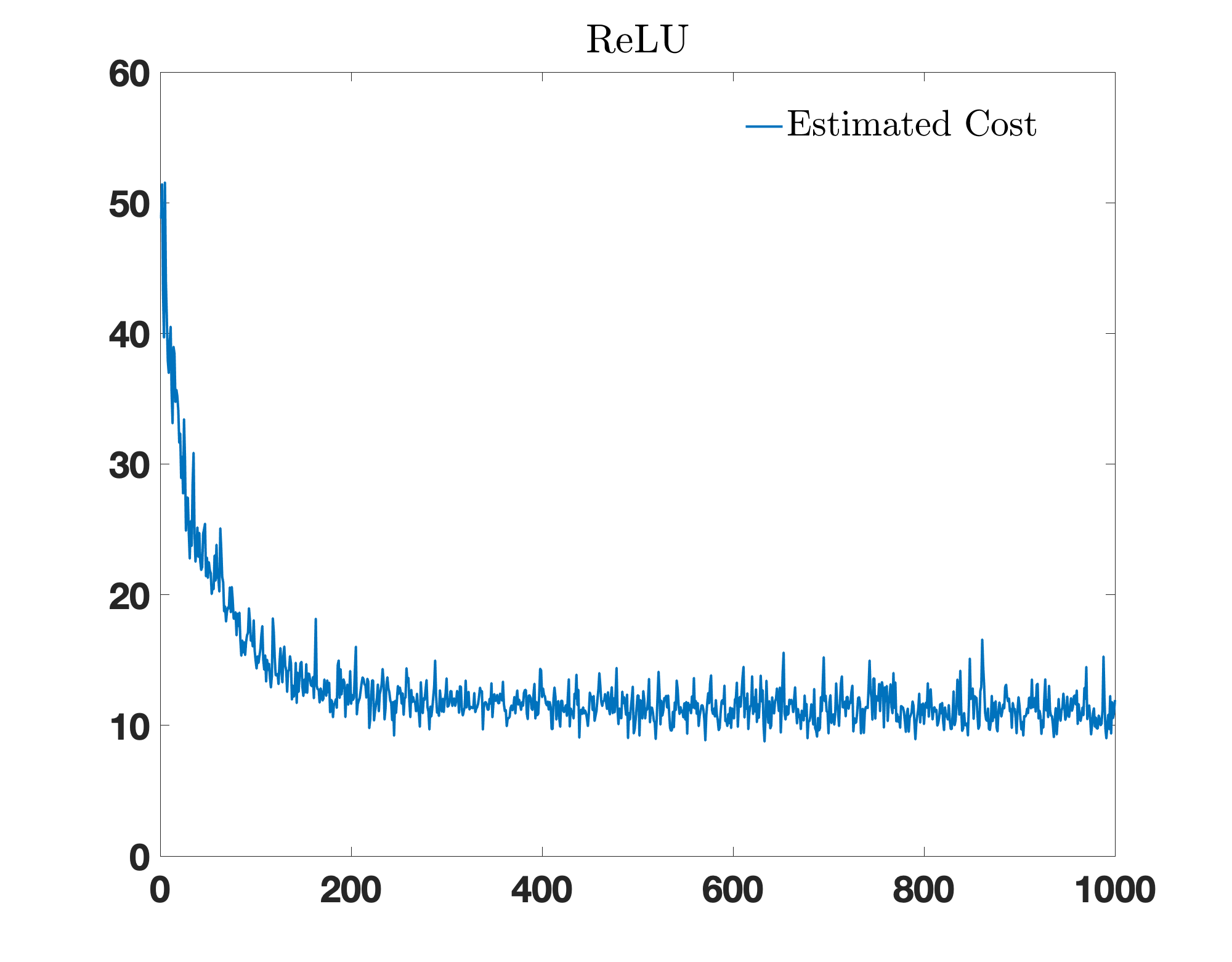}\quad \includegraphics[width=3in]{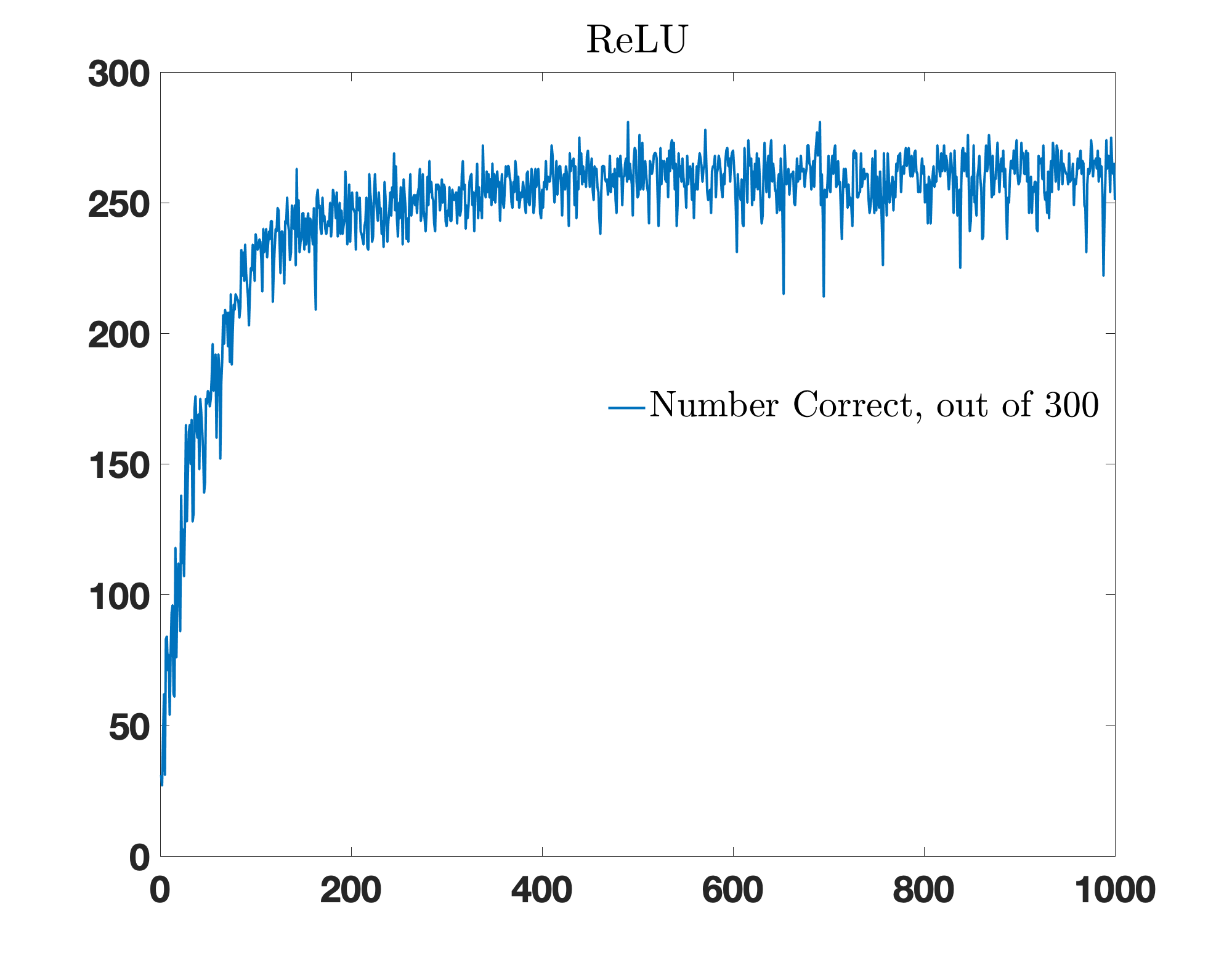}
    \caption{Performance of the  ReLU cost function (i.e., $h=0$).  Left image: estimate of the cost function over each iteration of the neural network (from 300 randomly selected elements from the MINST dataset).  Right image:  total number of images from the 300 whose digits were correctly identified.  For each image the $x$-axis represents the iteration number of the learning process. }
    \label{fig:ReLUPlots}
\end{figure}

We now demonstrate the learning of the reaction network in a different manner: by visualizing the output trajectories of a subset of the nodes on a particular image from the database, but after a different number of iterations of the learning process.  We arbitrarily chose the 30th image in the database, which is the 7 presented as the right-most image  in Figure \ref{fig:digit}.  Note that since the image is that of a 7, we hope and expect that the equilibrium value associated with the 8th output node of our system will eventually converge towards 10, whereas the values of the other output nodes will converge towards 1.  See Figure \ref{fig:changes} for trajectories of output nodes 1, 2, 6, and 8 (associated with the digits 0, 1, 5, and 7), and hidden nodes 1 and 32.  As expected, the equilibrium value associated to the 8th output node does indeed separate from the others and moves towards 10, as the number of iterations increases, whereas the other output nodes remain near the value 1.  Also of interest is that the equilibrium value associated with output node 2, which is associated with the digit 1, converges towards 1 slower than do the other output nodes.  We assume this is because the digit, which is a seven, has characteristics similar to the digit 1.  For the purposes of this particular calculation, our initial condition for all 50 nodes was chosen to be equal to one.

\begin{figure}
    \centering
    \includegraphics[width=3.0in]{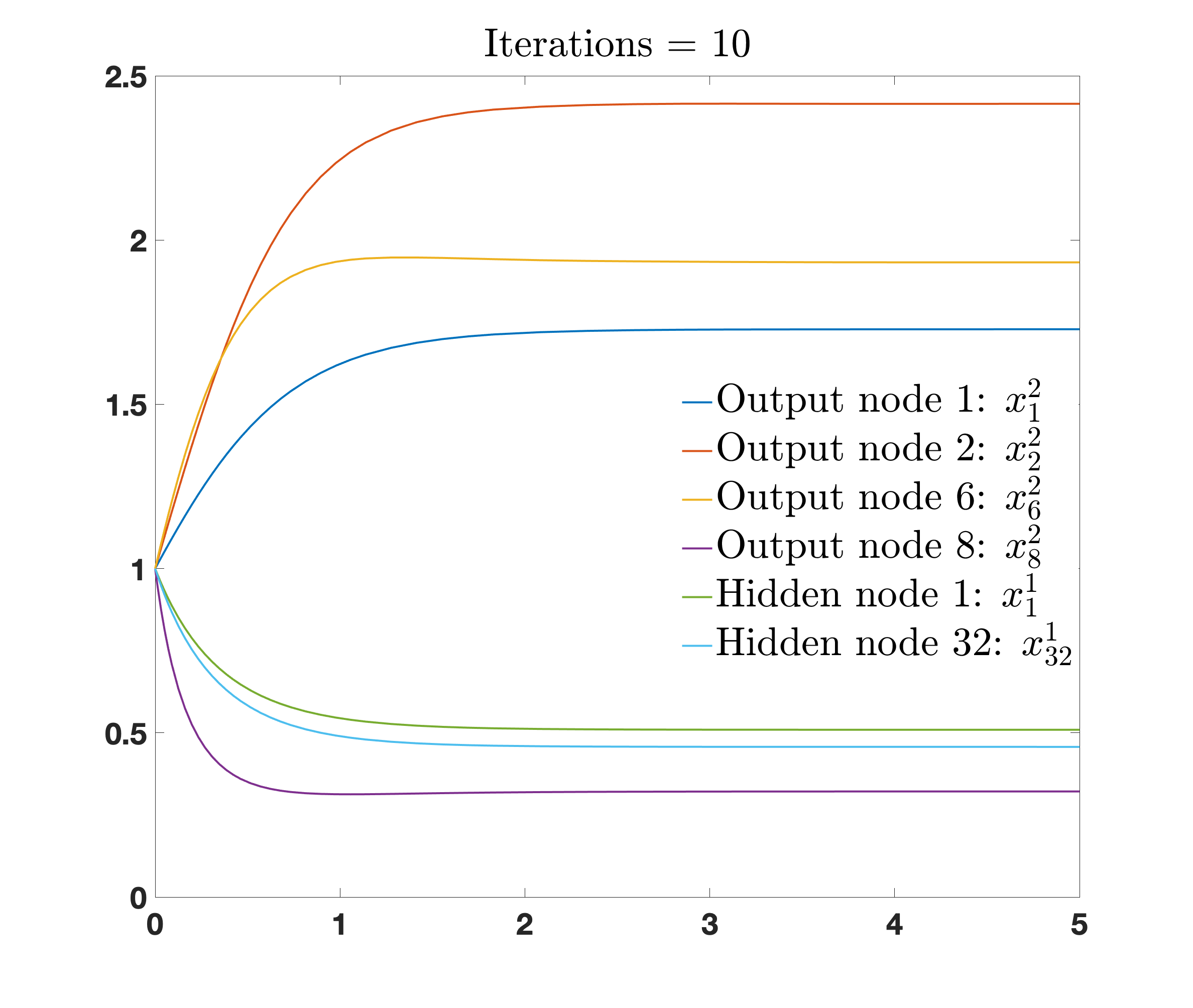}\quad \includegraphics[width=3.0in]{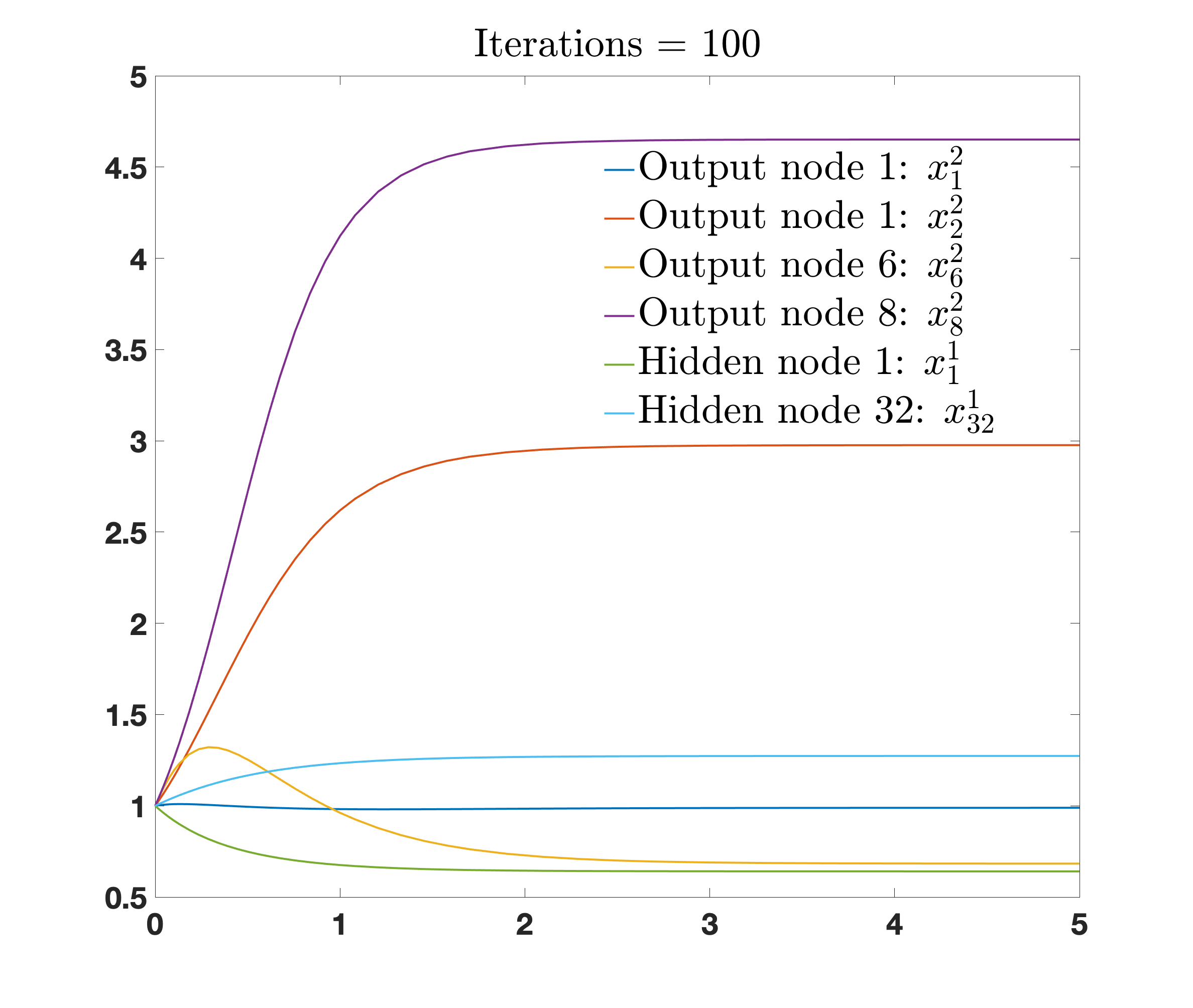}\\
    \includegraphics[width=3.0in]{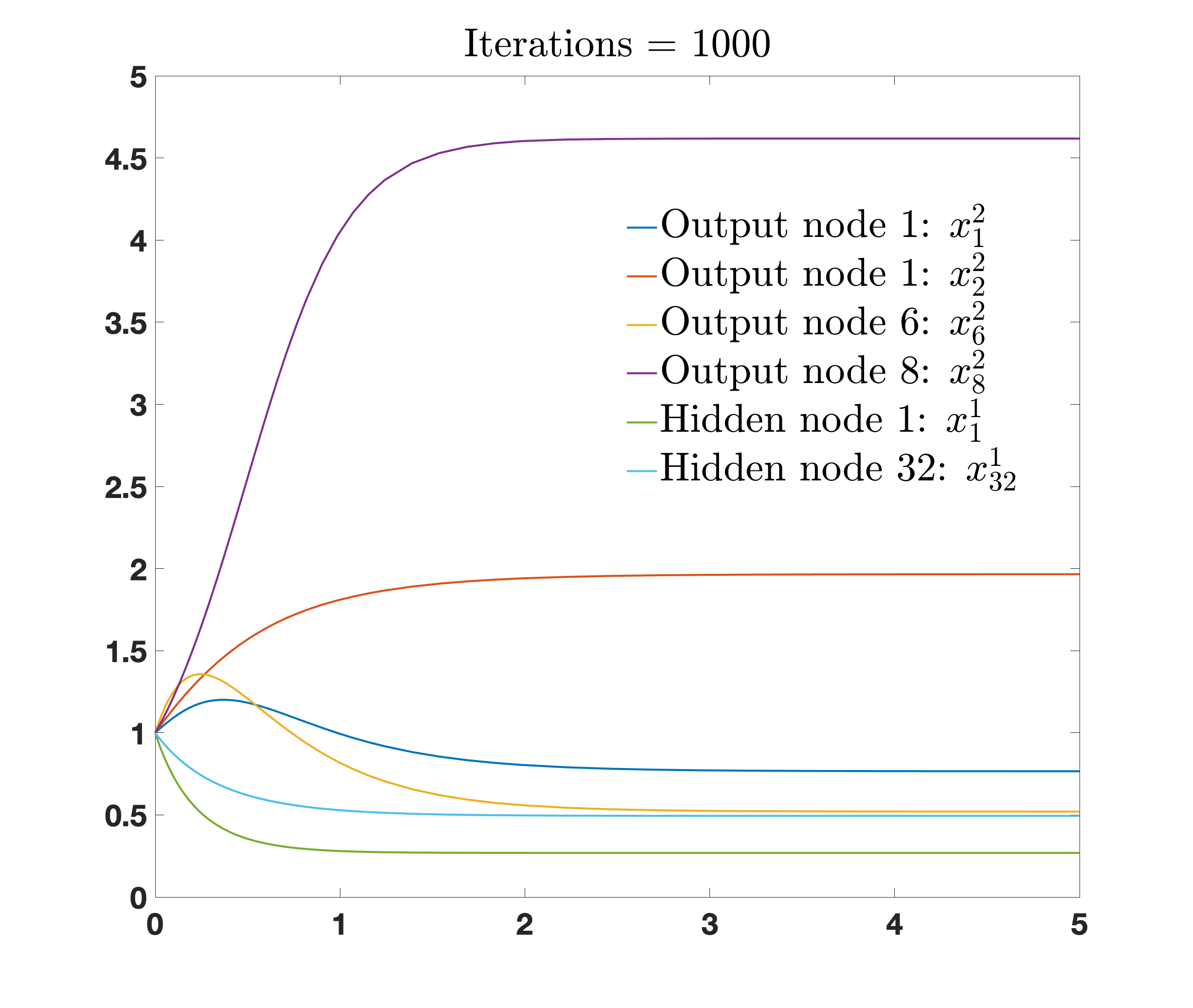}\quad \includegraphics[width=3.0in]{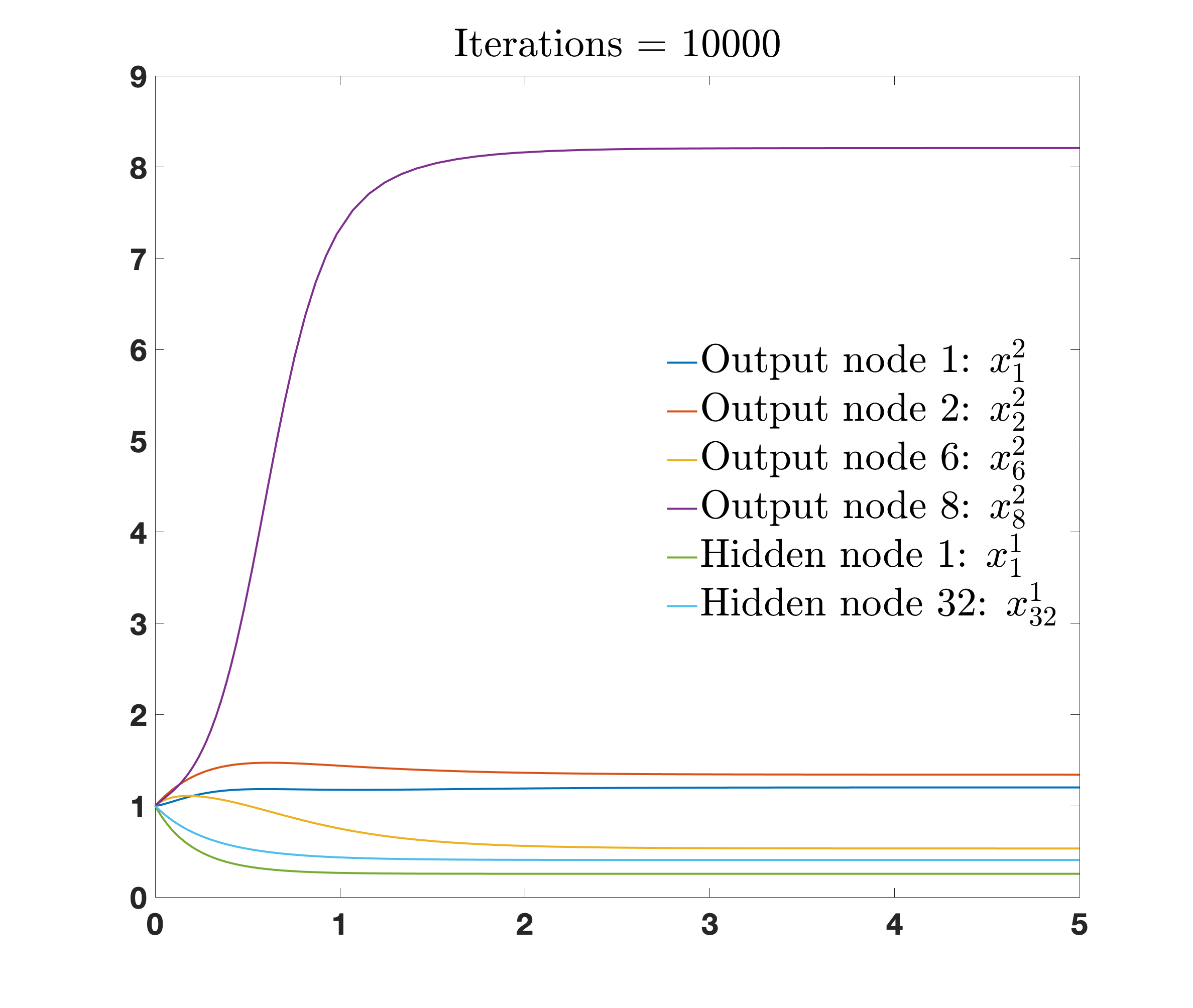}
    \caption{Plots of trajectories from the ODEs associated with our reaction network construction after different numbers of iterations.  Note how the equilibrium of the 8th output node, which is associated with the correct digit of 7, seems to be converging towards 10 as the number of iterations increases, whereas the equilibria associated with the other output nodes converge towards 1.}
    \label{fig:changes}
\end{figure}

As mentioned above, the fact that a neural network utilizing a ReLU activation function (smoothed or not) can be trained to identify the hand-drawn digits from the MNIST dataset is not the point of this paper, and is very well known.  Instead, we now focus on the behavior of the ODE associated with the reaction network implementation.  Using  the same setup as detailed above (including the randomized initial conditions), but with both the BatchSize and the number of iterations set to 1, we may output the values of the activations $a^\ell$ for the neural network with the smoothed ReLU activation function with $h = 1$.  There are a total of fifty terms (one for each node) in these vectors, which is too many to visualize.  We therefore arbitrarily selected the first and third nodes from the output layer and the first and 32nd nodes from the hidden layer.  The resulting values are
\begin{align*}
a_1^2 &= 1.54691661955827, \quad 
a_3^2 = 0.885658219979311,\\
a_1^1 &= 1.06208572398989, \quad a_{32}^1 = 0.72187173499449.  
\end{align*}
Next, we solved the system of 50 ODEs associated with our reaction network construction, with randomized initial conditions detailed above, while using exactly the same random variables as in the standard neural network implementation. 
We solved the resulting system of 50 ODEs, and representative plots for the chosen four nodes are given in Figure \ref{fig:repplots}.  We simulated until time 5 and found
\begin{align*}
x_1^2(5) &= 1.54703516476441, \quad x_3^2(5) = 0.885627963885228, \\
x_1^1(5) &= 1.06216260026126, \quad x_{32}^1(5) = 0.721901309123957.
\end{align*}
As our theory guaranteed, the values match those of the standard neural network very well.
\begin{figure}
    \centering
    \includegraphics[width=4in]{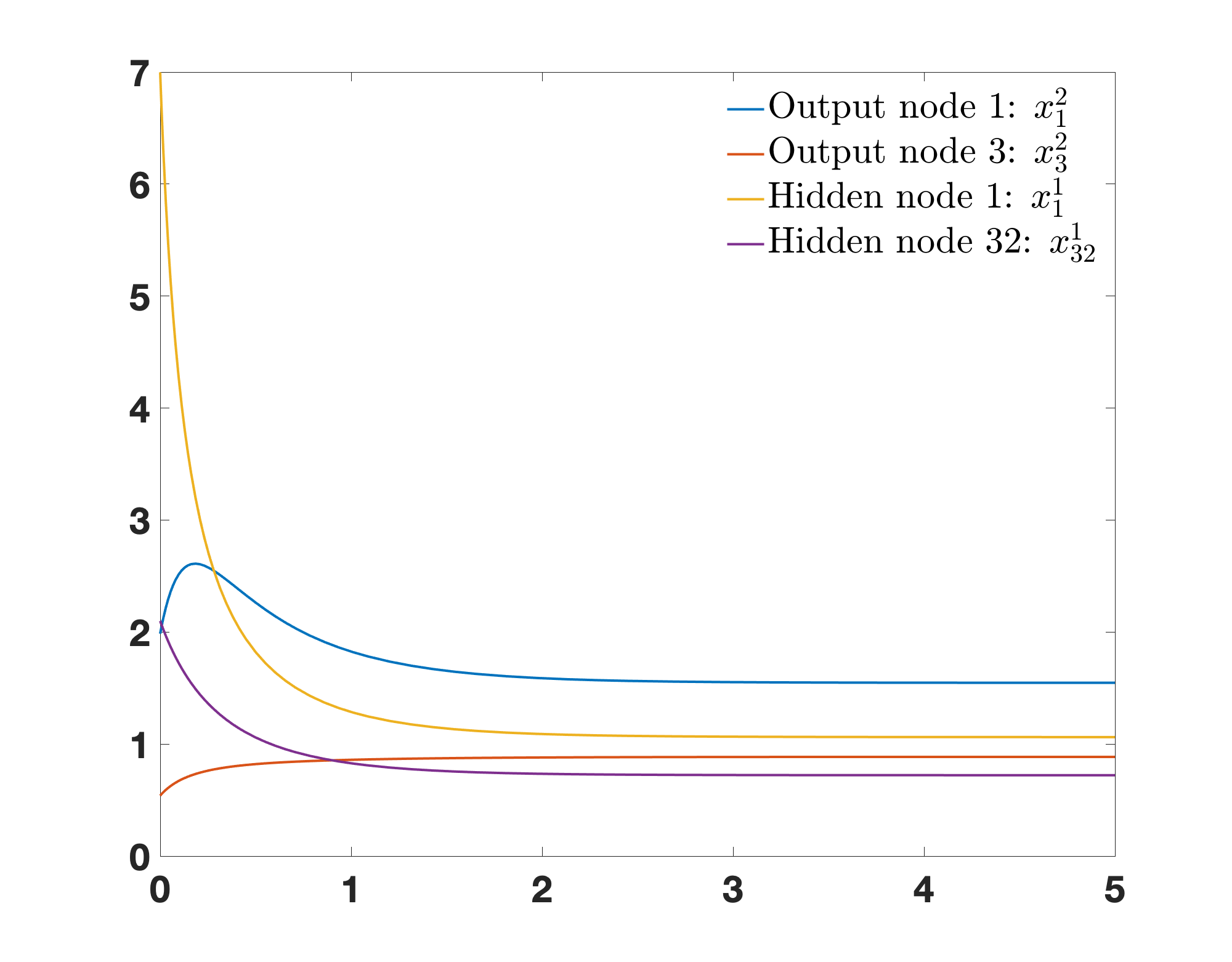}
    \caption{Representative plots with ``regularly sized'' initial conditions.  We chose to visualize nodes 1 and 3 from the output layer and nodes 1 and 32 from the hidden layer.}
    \label{fig:repplots}
\end{figure}

Of course, it is impossible to ``demonstrate'' convergence from infinity.  Instead, we simply modified the initial conditions to  
\begin{verbatim}
    x00=1000*rand([c1,1]);
\end{verbatim}
and performed the same ODE computations as detailed above.  See Figure \ref{fig:BigIC} for plots of the solutions.  The final values were
\begin{align*}
x_1^2(5) &= 1.54686939850725, \quad x_3^2(5) = 0.885624161108907, \\
x_1^1(5) &= 1.06213690351638, \quad x_{32}^1(5) = 0.721896022634959,
\end{align*}
which, again, match the values from the previous iterations.

\begin{figure}
    \centering
    \includegraphics[width=3.0in]{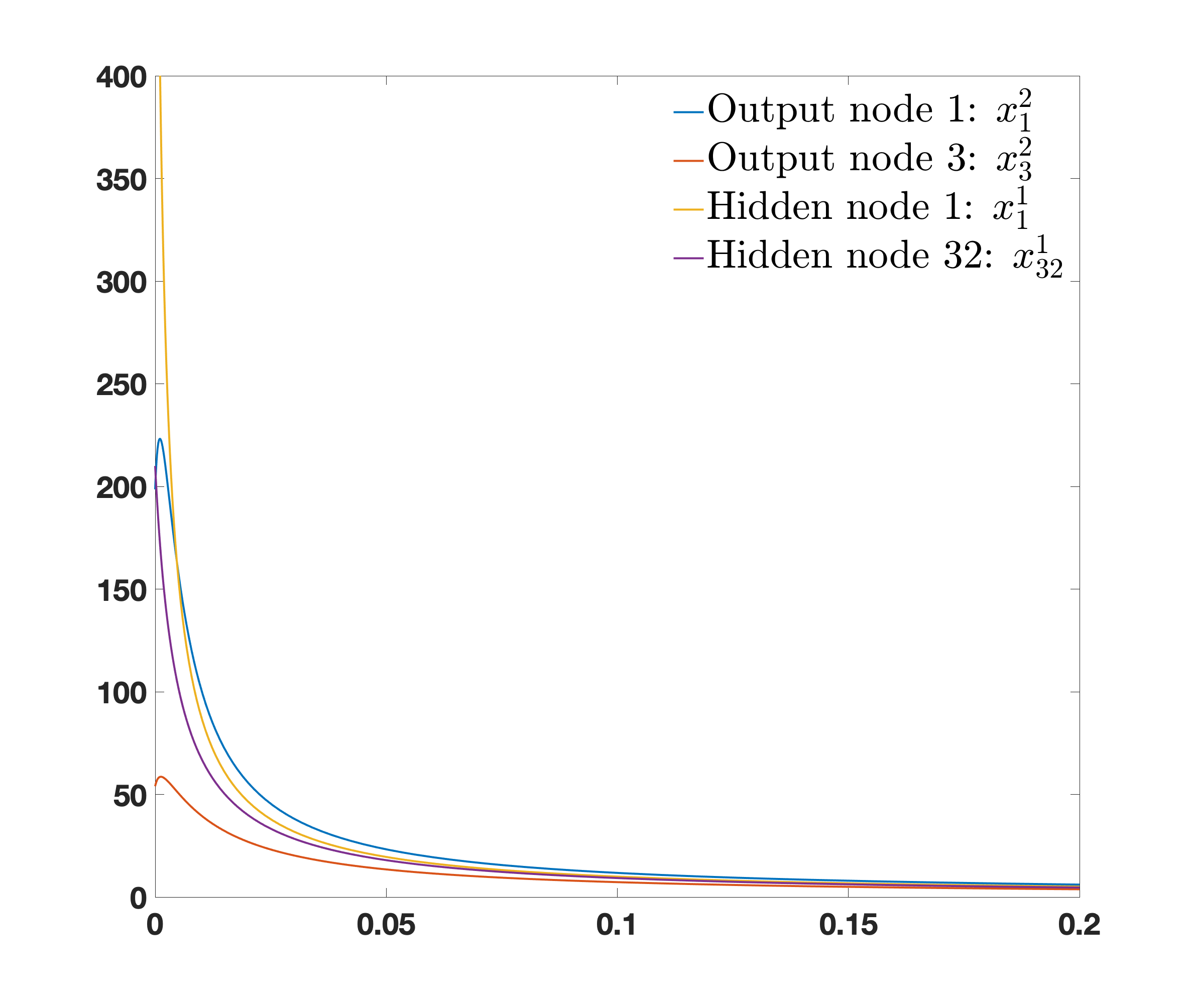}\quad \includegraphics[width=3.0in]{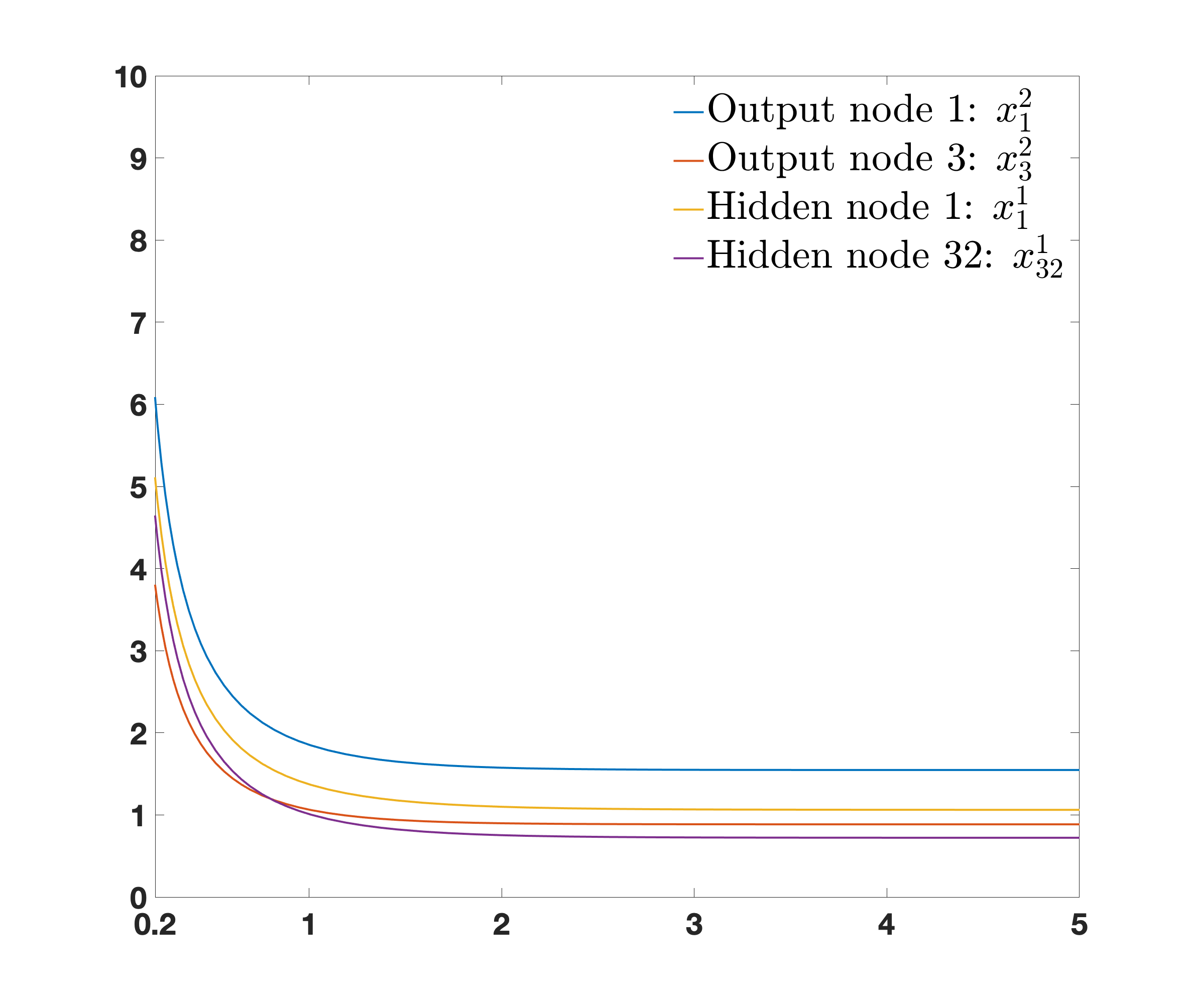}
    \caption{These are plots of the same trajectories. Note the scales on both the $x-$axes and the $y-$axes.  On the left we see the fast convergence down ``from infinity'' to values in the single digits.  On the right we see exponential convergence to the limiting  values. }
    \label{fig:BigIC}
\end{figure}

\subsection{Modifying the activation function}

We slightly modify the reaction network construction of Section \ref{sec:implement} by using $q = 3$ instead of $q = 2$ in  the final column of Table \ref{table:elementaryrxnnet}.  Thus, for each of the hidden and output nodes we simply change the reaction network so that it includes the reaction $3X \to X$ instead of $2X \to X$.  This change modifies the ODE for a particular node (hidden or output) to be
\begin{equation}\label{eq:newODE}
    \dot x = h + \rho\cdot x - 2x^3,
\end{equation}
with $h>0$ and $\rho$  as before.  Note that the above is the analog of $f_i^\ell$ in \eqref{eq:000998990}, and we are suppressing subscripts and superscripts for the sake of clarity (as we have done at times throughout the paper).   

 By Descartes' rule of signs, for each particular choice of $h$ and $\rho$ the system  \eqref{eq:newODE} has precisely one positive fixed point.    As can be shown by standard methods, this fixed point is stable.  Fixing $h>0$, the activation function for the resulting system is found by solving for the unique positive fixed point as a function of $\rho$.  See Figure \ref{fig:newactivation} for a plot of this activation function when $h =1$.  We see that the function is monotonic, grows like $\sqrt{y/2}$, as $y \to \infty$, and converges to zero, as $y \to -\infty.$ 
 Finally, it can be shown by similar arguments as in the proof of Proposition \ref{prop:our_network_NEW} that the resulting chemical system implements, in the sense of Definition \ref{def:implement}, the hardwired feed-forward neural network $(G,\PP,\varphi)$ where $\varphi$ is given in Figure \ref{fig:newactivation}, and that the system  converges from infinity in finite time (due to it having 3-polynomial decay) and is exponentially reliable.  
 
 \begin{figure}
    \centering
    \includegraphics[width=3.3in]{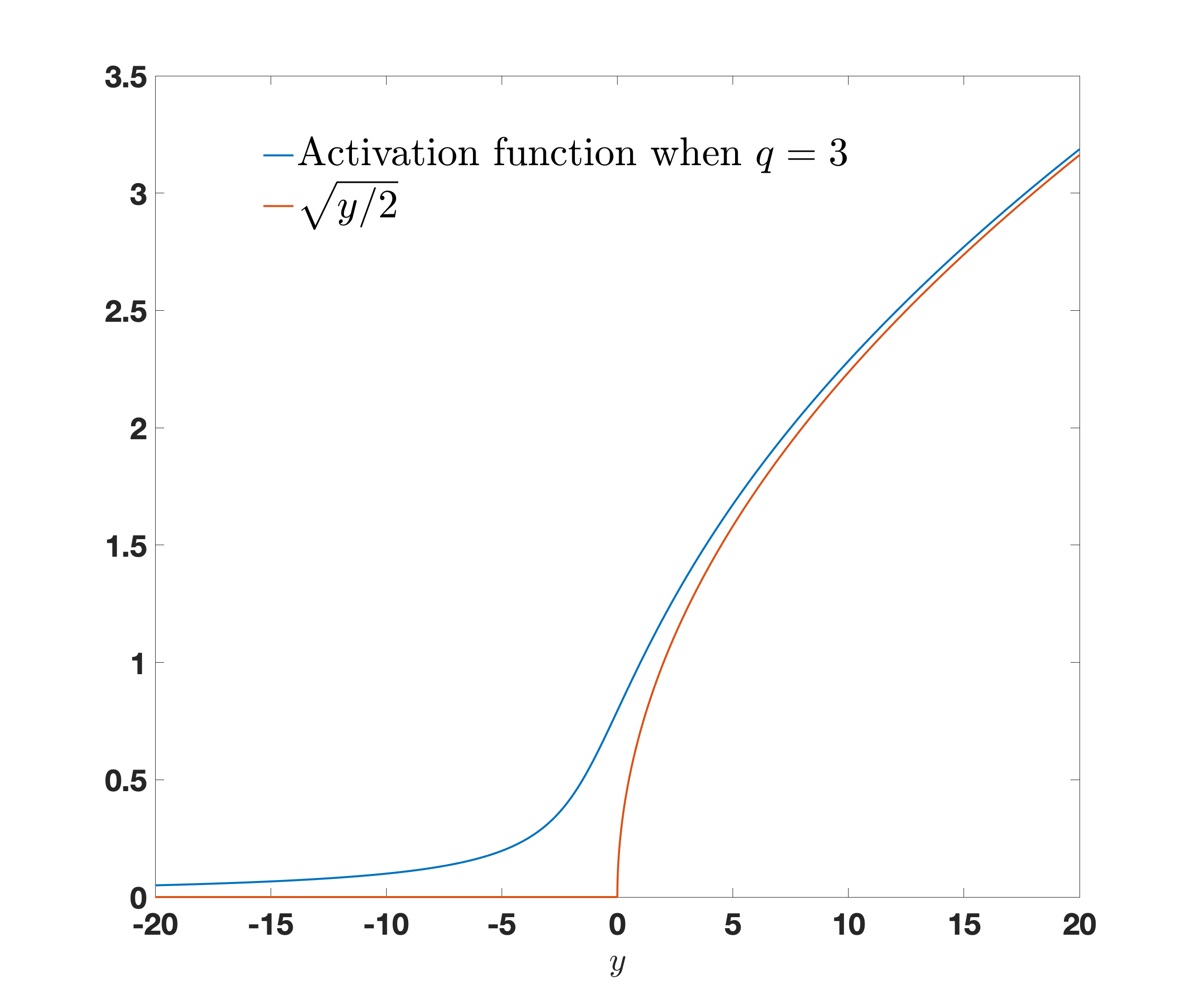}
    \caption{Activation function $\varphi$ implemented by the reaction network from Table \ref{table:elementaryrxnnet} with $q = 3$ and $h = 1$. This function is defined as the map between $y$ and the unique positive fixed point of the polynomial $1 + yx - 2x^3$.    A plot of $\sqrt{y/2}$ is added for the sake of comparison, which would be the corresponding activation function if $h$ were taken to be zero while $q = 3$. 
    }
    \label{fig:newactivation}
\end{figure}

 In order to implement this chemical system, we solved the associated ODEs, as detailed above.  However, in the case when $q = 2$ we had a nice analytic formula for the activation function, $\varphi$, given by \eqref{eq:modifiedReLU},  which we could easily differentiate  to find $\varphi'(z)$, and plug that expression into the relevant terms in \eqref{eq:allthederivatives} for the purposes of gradient descent.    In this case, we are not so fortunate.  However, this derivative can be calculated in a straightforward manner.  For a fixed value of $z$, we may denote the unique positive fixed point of \eqref{eq:newODE}, with $z =\rho$, via $\varphi(z)$, in which case we have that $\varphi(z)$ is defined implicitly via
 \[
    0=h + z\cdot \varphi(z) - 2\varphi(z)^3.
 \]
 Differentiating with respect to $z$ and solving yields
 \[
       \varphi'(z) = -\frac{\varphi(z)}{z - 3\cdot 2 \cdot \varphi(z)^2},
 \]
 As $\varphi(z)$ is the output from the ODE solver, we also get the derivative in a straightforward manner.  
 
 With all the details in place, we can run the system and implement the neural network via our new chemical reaction network.  
 In Figure \ref{fig:NewActivationPlots}, we provide plots of the estimated cost and the number of images correctly identified, out of a batch of 300, using this new chemical system and activation function when all other variables (i.e., numbers of layers, hidden nodes, seed of the random number generator, etc.) are kept the same as above.  We note that this activation function performs similarly to the ReLU activation function (see Figure \ref{fig:ReLUPlots}).

\begin{figure}
    \centering
    \includegraphics[width=3in]{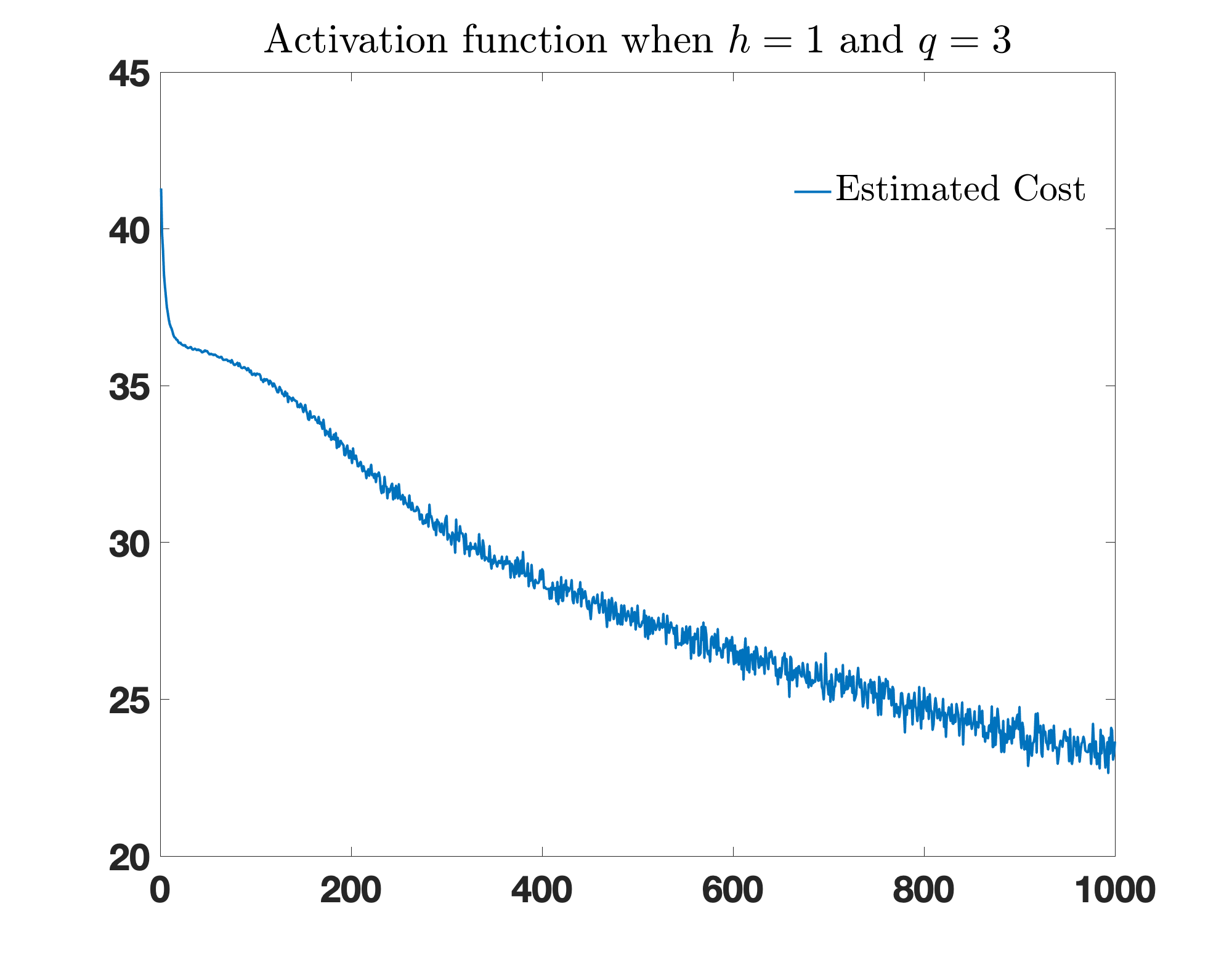}\quad \includegraphics[width=3in]{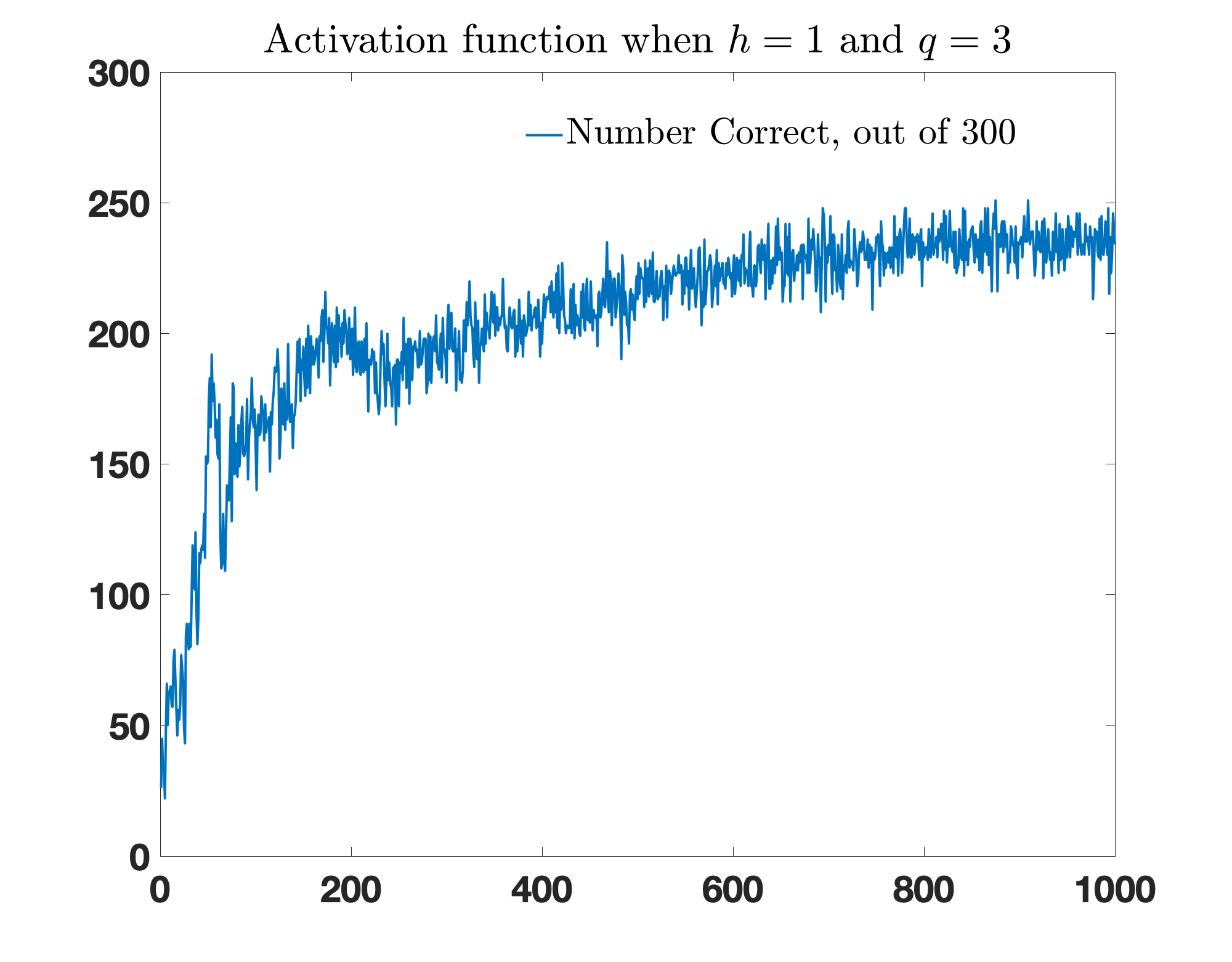}
    \caption{Performance of the  new activation  function (when $q = 3$).  Left image: estimate of the cost function over each iteration of the neural network (from 300 randomly selected elements from the MINST dataset).  Right image:  total number of images from the 300 whose digits were correctly identified.  For each image the $x$-axis represents the iteration number of the learning process. }
    \label{fig:NewActivationPlots}
\end{figure}

Finally, we note that we could also select $q$ to be any integer greater than 3, and a similar analysis can be carried out.  In particular, when $q$  is an integer  greater than or equal to $2$, we get an activation function that grows like $y^{1/(q-1)}$.  Moreover, the derivative can be calculated as above, and found to satisfy
\[
\varphi'(z) = - \frac{\varphi(z)}{z - q(q-1)\varphi(z)^{q-1}}.
\]
These systems with different activation functions could be useful in different settings.

\vspace{.15in}

\noindent {\Large \textbf{Acknowledgements}}

\vspace{.15in}

\noindent D.~Anderson gratefully acknowledges support  via the Army Research Office through grant W911NF-18-1-0324, and via the William F. Vilas Trust Estate. 
We thank Erik Winfree for helpful comments with an early draft of this paper.

\appendix
\section{Appendix}
\label{appendix}
\begin{prop}\label{prop:6587567}
Consider the ODE $\dot u = -cu^q$ where $c>0$, $u \in \R_{\ge 0}$, and $q \in \R$. If $q>1$, then $u(t) \le 1$ for any $t > ((q-1)c)^{-1}$ and any $u(0) = u_0 \in \R_{\ge 0}$.
\end{prop}
\begin{proof}
For $q > 1$, the function $-cu^q$ is  locally Lipschitz for  $u \in \R_{\ge 0}$ and so the initial value problem with $u(0) = u_0 \in \R_{\ge 0}$ has a unique solution which can be found by separation of variables: 
\[
    u(t) =\frac{1}{\left((q-1) c t + u_0^{-(q-1)}\right)^{1/(q-1)}}
\]
for all $t \in \R_{\ge 0}$. Clearly, $u(t) \xrightarrow{t \to \infty} 0$ monotonically for any $u_0 \in \R_{\ge 0}$. 
It suffices to assume that $u_0 > 1$. Define $t_1$ to be the time for which $u(t_1) =1$.  Then, since $u_0>1$, we have $(q-1) c t_1 = 1 - u_0^{1-q} \in (0,1)$, implying
\[
    t_1 = \frac{1}{(q-1)c}(1 - u_0^{1-q}) \in \left( 0,  \frac{1}{(q-1)c}\right).
\]  
Noting the monotonicity of $u(t)$ now finishes the proof.
\end{proof}

We provide a version of Gr\"onwall's inequality~\cite{bellman1943stability}. 
\begin{lemma}[Gr\"onwall's inequality]\label{thm:gronwall}
 Consider the interval $I=[t_0,t]$. Let $\alpha:I\to\mathbb{R}$ and $\beta:I\to\mathbb{R}$ be continuous functions. Let $V:I\to\mathbb{R}$ be a continuously differentiable function satisfying 
\begin{eqnarray}
\frac{d}{dt} V(t) \leq \alpha(t)V(t) + \beta(t)\,\, \mbox{ for }\, t\in I.
\end{eqnarray}
Let $V(t_0)=V_0$. Then,
\begin{eqnarray}
V(t) \leq V_0\exp \left(\int_{t_0}^t \alpha(s)ds\right) + \int_{t_0}^t \exp\left({\int_s^t \alpha(r)dr}\right)\beta(s)ds\,\, \mbox{ for }\, t\in I.
\end{eqnarray}

\end{lemma}

\bibliographystyle{amsplain}
\bibliography{Bibliography}

\end{document}